\documentclass[letterpaper]{article} 
\usepackage[preprint]{aaai2027}  
\usepackage[hyphens]{url}  
\usepackage{graphicx} 
\usepackage{natbib}  
\usepackage{caption} 
\usepackage{multirow}
\usepackage{algorithm}
\usepackage{algorithmic}

\newcommand{\green}[1]{#1}
\newcommand{\bon}[1]{#1}

\usepackage{newfloat}
\usepackage{listings}
\DeclareCaptionStyle{ruled}{labelfont=normalfont,labelsep=colon,strut=off} 
\floatstyle{ruled}
\newfloat{listing}{tb}{lst}{}
\floatname{listing}{Listing}

\usepackage{booktabs}

\usepackage{amsmath}
\usepackage{amssymb}
\usepackage{amsthm}
\usepackage{mathtools}
\usepackage{xcolor}
\usepackage{colortbl}
\usepackage{pifont}
\usepackage{framed}
\usepackage{tikz}
\usetikzlibrary{arrows.meta,positioning,fit,backgrounds,calc,shapes.geometric}
\graphicspath{{../results/M_sweep/figures/}{../results/figures/}{./}}
\newcommand{\xt}{x_t}
\newcommand{\xs}{x_s}
\newcommand{\Psith}{\Psi^\theta_{t\to s}}
\newcommand{\wk}{w^\theta_k}
\newcommand{\Pik}{P^\theta_i}
\newcommand{\Mask}{\texttt{[MASK]}}
\newcommand{\I}{\mathrm{I}}
\newcommand{\TC}{\mathrm{TC}}
\newcommand{\HH}{\mathrm{H}}
\newcommand{\E}{\mathbb{E}}
\newcommand{\R}{\mathbb{R}}

\newcommand{\Ihat}{\widehat{\I}_k}

\newcommand{\TChat}{\widehat{\TC}}
\DeclareMathOperator*{\logsumexp}{logsumexp}
\DeclareMathOperator*{\softmax}{softmax}

\newtheorem{theorem}{Theorem}
\newtheorem{proposition}{Proposition}
\newtheorem{lemma}{Lemma}
\newtheorem{corollary}{Corollary}

\newtheorem{remark}{Remark}

\definecolor{shadecolor}{gray}{0.93}
\let\origtheorem\theorem     \let\origendtheorem\endtheorem
\renewenvironment{theorem}{\begin{snugshade}\origtheorem}{\origendtheorem\end{snugshade}}
\let\origproposition\proposition \let\origendproposition\endproposition
\renewenvironment{proposition}{\begin{snugshade}\origproposition}{\origendproposition\end{snugshade}}
\let\origlemma\lemma         \let\origendlemma\endlemma
\renewenvironment{lemma}{\begin{snugshade}\origlemma}{\origendlemma\end{snugshade}}
\let\origcorollary\corollary \let\origendcorollary\endcorollary
\renewenvironment{corollary}{\begin{snugshade}\origcorollary}{\origendcorollary\end{snugshade}}

\newenvironment{gensample}[1]%
{%
 \MakeFramed{\advance\hsize-\width\FrameRestore}%
 \noindent\small{\textbf{#1}}\par\nobreak\smallskip\noindent}%
{\endMakeFramed}

\title{Latent-Kernel Discrete Flow Maps for Few-Step Generation}

\title{Latent-Kernel Discrete Flow Maps for Few-Step Generation}
\author{
    Mansoor Ahmed\textsuperscript{\rm 1,2}\corresponding,
    Yue-Tsz Fan\textsuperscript{\rm 2},
    Hemanth Venkateswara\textsuperscript{\rm 1},
    Murray Patterson\textsuperscript{\rm 1}\corresponding
}
\affiliations{
    \textsuperscript{\rm 1}Georgia State University, Atlanta, GA, USA\\
    \textsuperscript{\rm 2}Georgia Institute of Technology, Atlanta, GA, USA\\
    mahmed76@student.gsu.edu, mpatterson30@gsu.edu
}

\begin{document}

\maketitle

\begin{abstract}
Discrete diffusion and flow-matching models denoise a sequence over many steps, but to keep each step cheap, they factorize the transition across positions and decide every token \textit{independently}. This makes few-step generation challenging for text when the target couples two positions, such as a subject and a verb that must agree. An independent update commits to them separately, and many function evaluations are spent repairing the mismatch. Existing few-step methods buy back the lost correlation by distilling or rectifying a slow teacher, and so inherit the teacher's quality ceiling. We ask instead \textit{whether a model can express correlated steps natively}, and answer with Latent-Kernel Discrete Flow Maps (LKF), \textit{a from-scratch flow-map kernel that is a mixture of} $M$ \textit{factorized components tied by a single shared latent}. Conditioned on the latent, each component is cheap, and the mixture is summed over the latent in closed form for small $M$. We show that a single step places mass on correlated completions with the same sampling time complexity as a factorized model, since one latent is drawn per sequence and reused across the entire denoising trajectory. We also show that the Masked Diffusion Language Model (MDLM) is a special case of our LKF model at $M{=}1$. The experiments for unconditional text generation on the One-Billion-Word (LM1B) and WikiText-103 benchmarks show that our LKF model learns strongly heterogeneous components and improves generative perplexity by $2.1\times$ to $3.3\times$ over the likelihood baselines without losing diversity. The gain grows with $M$, and at $M{=}8$, it surpasses distilled and rectified few-step samplers. The source code is available at: \url{https://github.com/mansoor181/lkf.git}

\end{abstract}


\section{Introduction}

Discrete diffusion and flow-matching models now rival autoregressive language models, as they fill
every position of a sequence in parallel rather than one token at a
time~\cite{potaptchik2026dfm}. To keep a single denoising step cheap, they factorize the per-step
transition across positions,
\begin{equation}
p_{t\to s}(\xs\mid \xt)\;=\;\prod_{i=1}^{L} p_i(\xs^i\mid \xt),
\label{eq:factorized}
\end{equation}
and decide every token \textit{independently} given the current sequence
$\xt$~\cite{sahoo2024mdlm}. This independence is what
makes a step cheap, and it is also what makes few-step generation challenging when the target
couples positions that must be decided together. If a sentence pairs a subject with a verb that
must agree in number, or an opening bracket with its closing match, the correct tokens are
correlated, and an independent update commits to them separately.
{A factorized step may match the individual token marginals, while assigning
substantial probability to jointly inconsistent combinations, and additional denoising steps may
therefore be needed to resolve the mismatch.}

\begin{figure*}[ht!]
\centering
\begin{tikzpicture}[
  ctx/.style={draw,gray!120,rounded corners=3pt,fill=gray!8,
              minimum height=6mm,inner xsep=7pt,font=\small},
  marg/.style={draw,violet!60!black,rounded corners=3pt,fill=violet!12,
              minimum height=6mm,minimum width=30mm,inner xsep=5pt,font=\small},
  lat/.style={draw,orange!75!black,rounded corners=3pt,fill=orange!18,
              minimum height=6mm,minimum width=30mm,inner xsep=5pt,font=\small,align=center},
  bad/.style={draw,red!65!black,rounded corners=3pt,fill=red!10,
              minimum height=6mm,minimum width=30mm,inner xsep=5pt,font=\small},
  good/.style={draw,green!45!black,rounded corners=3pt,fill=green!12,
              minimum height=6mm,minimum width=30mm,inner xsep=5pt,font=\small},
  lab/.style={font=\scriptsize,gray!40!black},
  tl/.style={font=\small\bfseries},
  flow/.style={-{Stealth[length=2mm]},gray!60!black,semithick},
]
\node[tl,anchor=west] at (0,3.1) {(a) Factorized step (\textit{other methods})};
\node[ctx]  (ctxa) at (3.3,2.4)  {The \texttt{[Mask$_1$]} \texttt{[Mask$_2$]} loudly.};
\node[marg] (m1)   at (1.7,1.15) {dog .5 $|$ dogs .5};
\node[marg] (m2)   at (4.9,1.15) {barks .5 $|$ bark .5};
\draw[flow] (ctxa.south -| m1) -- (m1.north);
\draw[flow] (ctxa.south -| m2) -- (m2.north);
\node[lab] at (3.3,1.75) {independent draws};
\node[good] (oa1) at (1.7,0.0)   {dog barks \ding{51} $\tfrac14$};
\node[bad]  (oa2) at (4.9,0.0)   {dog bark \ding{55} $\tfrac14$};
\node[bad]  (oa3) at (1.7,-0.85) {dogs barks \ding{55} $\tfrac14$};
\node[good] (oa4) at (4.9,-0.85) {dogs bark \ding{51} $\tfrac14$};
\draw[flow] (m1.south) -- (oa1.north);
\draw[flow] (m2.south) -- (oa2.north);
\node[lab,align=center] at (3.3,-1.7)
  {$P(\text{agree})=\tfrac12$ in one factorized draw\\
   half of the joint mass is inconsistent};
\draw[gray!40,thin] (6.9,3.3) -- (6.9,-2.1);
\begin{scope}[shift={(7.5,0)}]
\node[tl,anchor=west] at (0,3.1) {(b) LKF step (\textit{ours}), $M=2$};
\node[ctx] (ctxb) at (3.3,2.4) {The \texttt{[Mask$_1$]} \texttt{[Mask$_2$]} loudly.};
\node[lat] (k1) at (1.7,1.1) {$\green{k=1}$: \emph{singular}\\ $w_1=.5$};
\node[lat] (k2) at (4.9,1.1) {$\green{k=2}$: \emph{plural}\\ $w_2=.5$};
\draw[flow] (ctxb.south -| k1) -- (k1.north);
\draw[flow] (ctxb.south -| k2) -- (k2.north);
\node[lab] at (3.3,1.9) {$\green{k\sim w(\xt)}$, drawn once};
\node[marg] (c1) at (1.7,0.0)  {dog 1 $\cdot$ barks 1};
\node[marg] (c2) at (4.9,0.0)  {dogs 1 $\cdot$ bark 1};
\draw[flow] (k1.south) -- (c1.north);
\draw[flow] (k2.south) -- (c2.north);
\node[good] (ob1) at (1.7,-0.85) {dog barks \ding{51} $\tfrac12$};
\node[good] (ob2) at (4.9,-0.85) {dogs bark \ding{51} $\tfrac12$};
\draw[flow] (c1.south) -- (ob1.north);
\draw[flow] (c2.south) -- (ob2.north);
\node[lab] at (3.3,-1.7) {$P(\text{agree})=1$ in one transition};
\end{scope}
\end{tikzpicture}
\caption{Two masked tokens must agree in number, and their correct completions are
therefore correlated. \textbf{(a)} A factorized transition (Eq.~\ref{eq:factorized}) samples
both positions independently from $.5/.5$ marginals. Its outer-product joint assigns
probability $\tfrac14$ to each completion, and only half of its mass therefore lies on
agreement outcomes. \textbf{(b)} An LKF transition (Eq.~\ref{eq:kernel}) with $M{=}2$ draws a
shared latent $\green{k\sim w(\cdot\mid\xt,t,s)}$. Conditioned on $\green{k}$, the positions remain
independent, but each component concentrates on a consistent completion. Marginalizing
$\green{k}$ preserves the same per-position marginals as in (a), while placing all joint mass on
the agreement set. Thus, two rank-one components express the correlated completion in
one transition.}
\label{fig:concept}
\end{figure*}
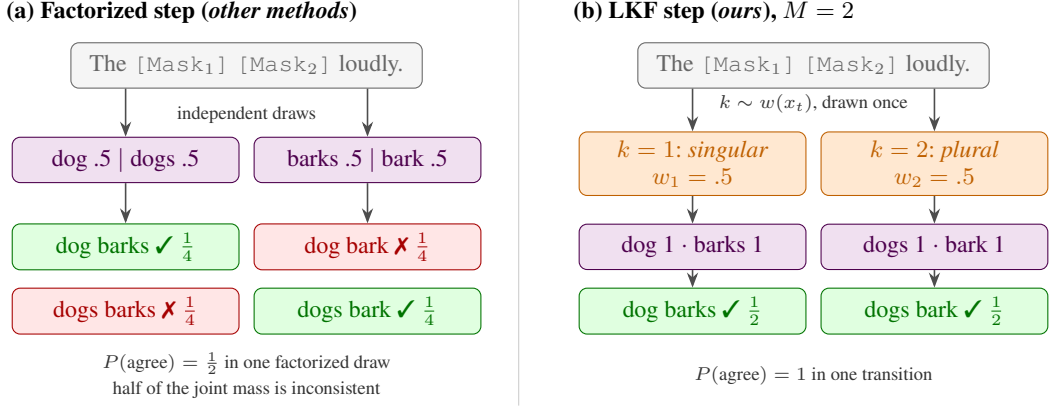

Currently, two model families dominate this trade-off. \textbf{Likelihood models} such as
MDLM~\citep{sahoo2024mdlm}
and SEDD~\citep{lou2023sedd} train a strong factorized denoiser from scratch and then spend many
steps at inference, where the factorization assumption is least harmful. \textbf{Few-step
accelerators} take a trained teacher and compress its many-step behaviour into a handful of steps,
either by \emph{distillation} (Di4C~\citep{hayakawa2024di4c}, SDTT~\citep{deschenaux2025sdtt}) or by
\emph{coupling rectification} (ReDi~\citep{yoo2026redi}). Both methods are effective, but they
inherit the teacher's quality ceiling. It is well understood that a student cannot outrun the
teacher it imitates~\cite{zhang2025can}, and rectification degrades after a few rounds. Hence,
neither family removes the factorization at its root.

We ask whether correlated steps can be modeled natively without a teacher. Our answer is
\textbf{Latent-Kernel Discrete Flow Maps} (LKF), whose finite-time kernel is a mixture of
$M$ factorized components selected by a latent $\green{k\in[M]}$:
\begin{equation}
\Psith(\xs\mid \xt)
=
\sum_{k=1}^{M}\wk(\xt,t,s)
\prod_{i=1}^{L}\Pik(\xs^i\mid \xt,k,t,s).
\label{eq:kernel}
\end{equation}
Conditioned on $(\xt,t,s)$, we first sample
$\green{k\sim w(\cdot\mid\xt,t,s)}$, and we then sample the coordinates of $X_s$
independently given $\green{k}$. Thus, $\green{k}$ couples positions after marginalization,
while each component remains factorized. At $M{=}1$, LKF reduces to MDLM.
This mechanism produces a teacher-free few-step gain.
On the One-Billion-Word (LM1B) and WikiText-103 benchmarks, the mixture learns
heterogeneous components, and a best-of-$M$ decode that ranks those components by the model's own
likelihood turns them into a few-step gain. 


Our main contributions are as follows.
\begin{enumerate}
\item We introduce a latent-kernel discrete flow-map (LKF) whose step is a mixture of
$M$ factorized components tied by one shared latent.
\item We give an exact information decomposition of the
correlation a rank-$M$ step can carry, turning the question of whether the
latent captures correlation into two bounded, estimator-validated quantities.
\item We test LKF on synthetic datasets with closed-form total correlation (TC). A
\textit{hidden-agreement} dataset shows that the captured information and the held-out
likelihood grow monotonically with $M$, while a \textit{parity control} shows that the latent
provably cannot help.
\item We evaluate unconditional generation on the LM1B and
WikiText-103 benchmarks against likelihood and accelerator models. At $M{=}8$, LKF improves
few-step generative perplexity by roughly $3\times$ over the likelihood baselines and surpasses
the distilled and rectified accelerators without using any teacher, with
the gain growing with $M$.
\end{enumerate}


\section{Related work}
\label{sec:related}

\textbf{Masked diffusion as likelihood models.} Building on the absorbing D3PM
process~\citep{austin2021d3pm}, MDLM~\citep{sahoo2024mdlm} shows that the absorbing evidence lower
bound (ELBO) reduces to a weighted cross-entropy on masked positions under the SUBS
parameterization, which zeroes the mask probability and copies revealed tokens through, and
SEDD~\citep{lou2023sedd} reaches comparable likelihoods by learning discrete score ratios for a
continuous-time Markov chain~\citep{campbell2022ctmc}. Both are factorized per step, where a rate
changes one coordinate at a time in SEDD, and the reverse kernel is a product over positions in
MDLM. Hence, both need many evaluations before the sample quality saturates. LKF reduces to MDLM
as its $M{=}1$ special case (Proposition~\ref{prop:contain}) by adding intra-step correlation.

\textbf{Few-step generation by distillation and rectification.}
SDTT~\citep{deschenaux2025sdtt} self-distills a masked model through time, and
Di4C~\citep{hayakawa2024di4c} distills into a mixture-of-products student, proving that
factorized students face a $\Theta(1/J)$ total-variation floor.
Hence, both methods are structurally teacher-bounded. ReDi~\citep{yoo2026redi} instead rectifies
the source-to-target coupling, provably decreasing the same conditional total correlation per
round, but its gains ride on a monotone entropy collapse. LKF asks the complementary question, whether the
mixture structure can be trained from scratch with no teacher ceiling, enriching the kernel until
the correlation is representable.

\textbf{Flow maps and from-scratch few-step models.} Flow maps learn a finite-time
transition kernel indexed by two times, where one network serves every step. Building on the
discrete flow matching family~\citep{gat2024dfm,campbell2024dfm}, Discrete Flow
Maps~\citep{potaptchik2026dfm} are the closest from-scratch sibling that carries
correlation through position-dependent simplex schedules, but with no latent and no mixture.
Concurrent theory quantifies the schedule-level tradeoff between total correlation and step
count~\citep{dmitriev2026sampling,zhao2026adaptation}, complementing our per-step accounting. LKF is a likelihood model and
differs by making each jump a low-rank mixture of discrete-token kernels whose latent is
explicit and measurable, a mechanism orthogonal to simplex schedules.

\section{Preliminaries and background}
\label{sec:background}

\textbf{Setup.} We model sequences $x=(x^1,\dots,x^L)$ whose tokens take one of $V$ values
from a vocabulary $\mathcal V$, and corrupted sequences additionally use an absorbing symbol
$\Mask\notin\mathcal V$. Data come from an unknown $p_{\text{data}}$ accessed only through
samples, and we pursue both generation and density evaluation on held-out data. Superscripts
index positions ($x^i$), and subscripts index time ($\xt$), with all notations collected in
Appendix~\ref{app:notation}. Samples are generated by iterative
denoising~\citep{austin2021d3pm}. Starting from all-$\Mask$, a learned transition removes a
little corruption per step, where each step is one network forward pass (NFE).

\textbf{Masked diffusion.} A monotone schedule $\kappa_t\in[0,1]$ keeps each token
independently with probability $\kappa_t$ and otherwise replaces it by $\Mask$~\citep{austin2021d3pm}
(Eq.~\eqref{eq:forward}). Masking is also called absorbing, where the reverse process turns $\Mask$ into
real tokens. A denoiser takes the whole partially masked sequence and predicts a categorical over
its clean value for each masked position, factorized across positions as in
Eq.~\eqref{eq:factorized}. MDLM~\citep{sahoo2024mdlm} and absorbing SEDD~\citep{lou2023sedd} are
such examples. The endpoint coupling is trivial ($\delta_{\Mask}\otimes p_{\text{data}}$), and
the only modelling freedom left is therefore the \emph{order and grouping} in which positions are
revealed, which the LKF latent parameterizes. Factorization is also what forces many steps.
Correlated positions revealed in the
same step are sampled independently and frequently disagree, and the process can repair the
inconsistency only by revealing them in many steps.

\textbf{Flow maps.} Discrete flow and diffusion models come in two flavours. \emph{Rate-based}
models (such as SEDD and CTMC discrete diffusion~\citep{campbell2022ctmc}) learn an instantaneous
transition rate and integrate it with a fine discretization at inference. \emph{Flow-map}
models learn a \emph{finite-time} kernel $\Psith(\xs\mid\xt)$ that jumps from $t$ to $s$ directly,
building on the discrete flow matching family~\citep{gat2024dfm,campbell2024dfm}. LKF is a flow-map
model, and MDLM is its special case with one component
(Proposition~\ref{prop:contain}, Appendix~\ref{app:objx0}).

\textbf{Total correlation.} Correlation is measured by the \emph{total correlation (TC)} of a joint
distribution over the $L$ positions,
\begin{equation}
\TC(X)\;=\;\sum_{i=1}^{L}\HH(X^i)\;-\;\HH(X)\;\ge\;0,
\label{eq:tc_def}
\end{equation}
the amount by which the joint exceeds the product of its marginals. It vanishes when the positions
are independent, and it is precisely what a factorized model discards.

\textbf{Why kernels and not rates.} Mixing several factorized \emph{rate} matrices does not
buy correlation. Convex combinations of coordinatewise CTMC generators are again coordinatewise,
so the mixed chain still changes one coordinate at a time and its short-time kernel is factorized
to $O(\varepsilon^2)$ (Proposition~\ref{prop:rates}, Appendix~\ref{proof:rates}). Correlated transitions exist
only in \emph{discrete-time} kernels or flow maps, which is why LKF is defined as a
mixture kernel (Eq.~\eqref{eq:kernel}) rather than a mixture of rates.

\section{Method}
\label{sec:method}

\subsection{The latent-kernel flow map (LKF)}

LKF parameterizes Eq.~\eqref{eq:kernel} with a single network that, given $(\xt,t,s)$, emits
(i) per-latent per-position token logits $\Pik(\cdot\mid \xt,k,t,s)$ and (ii) router logits
$\wk(\xt,t,s)$. Sampling one step draws a \emph{single} $k\sim w(\xt)$ shared across all
positions, and then samples each token from component $k$. The shared $k$ is the carrier of
intra-step correlation in LKF, and Figure~\ref{fig:sampling} in Appendix~\ref{app:arch}
illustrates this on a worked example.

\paragraph{Masking path and two-time transition}

We use a linear schedule $\kappa_t=t$ (data at $t{=}1$, all-\Mask{} at $t{=}0$). Forward noising
keeps each token with probability $\kappa_t$:
\begin{equation}
\xt^i=\begin{cases} x_1^i & \text{w.p. } \kappa_t,\\ \Mask & \text{w.p. } 1-\kappa_t.\end{cases}
\label{eq:forward}
\end{equation}
The \emph{two-time} training target for $0\le t<s\le 1$ reveals each still-masked position to its
data value with probability $p_{\text{rev}}=\frac{(\kappa_s-\kappa_t)}{(1-\kappa_t)}$ and leaves already
revealed positions unchanged:
\begin{equation}
\xs^i=\begin{cases}
\xt^i & \text{if } \xt^i\neq\Mask \quad(\text{identity}),\\
x_1^i & \text{if } \xt^i=\Mask,\ \text{w.p. } p_{\text{rev}},\\
\Mask & \text{if } \xt^i=\Mask,\ \text{w.p. } 1-p_{\text{rev}}.
\end{cases}
\end{equation}
The training loss counts only
masked positions because the kept positions transition deterministically ($\log 1=0$), and no gradient is spent on the copies.

\subsection{Architecture overview}

Figure~\ref{fig:lkfoverview} illustrates the training and inference process in LKF.
The network (Fig.~\ref{fig:arch} in Appendix~\ref{app:arch}) is a single transformer
trunk as in the diffusion transformer (DiT)~\cite{peebles2023dit}, and it outputs both the router
weights $w(\xt,t,s)\in\Delta^{M-1}$ and the full stack of per-latent factorized logits
$\{\Pik(\cdot\mid\xt,k,t,s)\}$ in one forward pass. Concretely, the network computes a tensor of
logits of shape $[B,M,L,V]$ (batch, latents, positions, vocab), router logits $[B,M]$, and
the slice $[:,k,:,:]$ (the factorized denoiser of component $k$).

Three ideas distinguish it from a vanilla masked DiT. The \textbf{two-time
conditioning} lets \emph{both} endpoints $t$ and $s$ modulate every block, which lets the same
weights realize a finite-time flow map for any step size. The \textbf{late latent injection} runs
the first $\text{depth}-L_k$ blocks once and replicates only the last
$L_k$ blocks over the $M$ latents, which brings the cost of marginalization down to
$M\,L_k/\text{depth}$ rather than $M\times$ the whole network. \textbf{A pooled router} reads the
shared representation and emits one mixture weight per sequence.

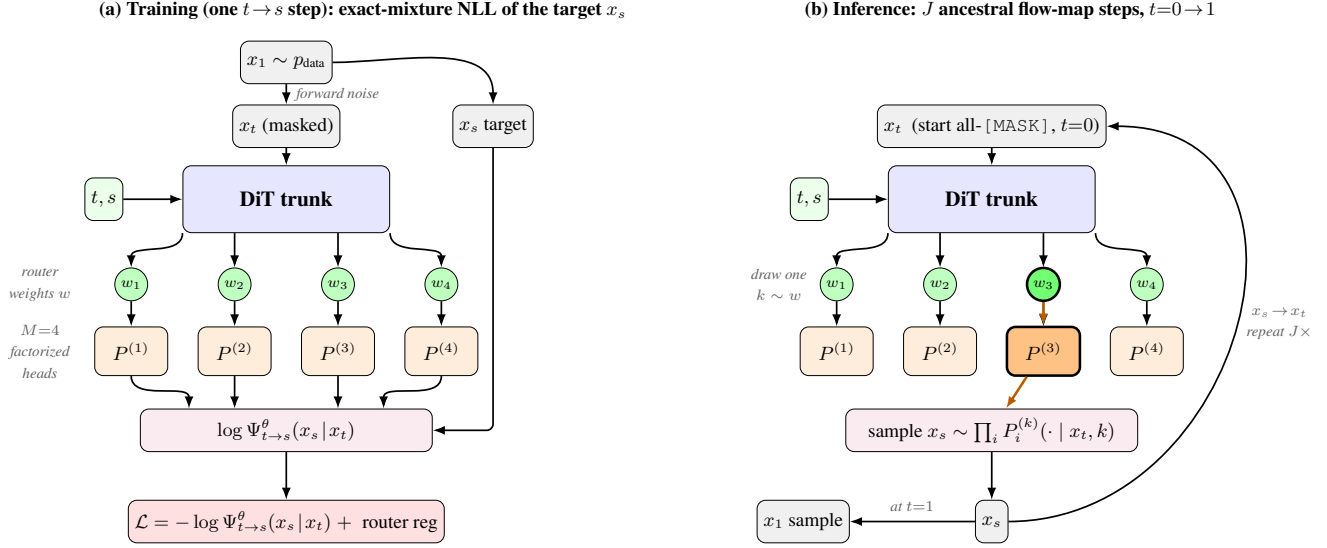
\begin{figure*}[t]
\centering
\resizebox{\textwidth}{!}{%
\begin{tikzpicture}[
  box/.style={draw,rounded corners,align=center,inner sep=3pt,font=\small,minimum height=7mm},
  io/.style={box,fill=gray!12},
  cnd/.style={box,fill=green!8},
  trunk/.style={box,fill=blue!10,minimum width=34mm,minimum height=11mm,font=\bfseries},
  head/.style={box,fill=orange!14,minimum width=12mm,minimum height=8mm},
  headsel/.style={head,fill=orange!48,very thick},
  comb/.style={box,fill=purple!8,align=center,minimum width=48mm},
  loss/.style={box,fill=red!12,align=center},
  wnode/.style={draw,circle,fill=green!25,inner sep=0.5pt,font=\scriptsize,minimum size=5.4mm},
  wsel/.style={wnode,fill=green!55,very thick},
  ar/.style={-{Latex[length=2mm]},thick},
  aro/.style={-{Latex[length=2mm]},very thick,orange!72!black},
  lbl/.style={font=\scriptsize\itshape,text=black!62,align=center},
  tl/.style={font=\small\bfseries},
]
\begin{scope}
\node[tl,anchor=west] at (-3.2,4.4) {(a) Training (one $t\!\to\!s$ step): exact-mixture NLL of the target $\xs$};

\node[io] (x1) at (0,3.55) {$x_1\sim p_{\text{data}}$};
\node[io] (xt) at (0,2.5) {$\xt$ (masked)};
\draw[ar] (x1)--(xt) node[midway,right=1pt,lbl]{forward noise};
\node[io] (xs) at (3.4,2.5) {$\xs$ target};
\draw[ar] (x1.east) to[out=0,in=90] (xs.north) node[midway,above=1pt,lbl]{};

\node[cnd] (cond) at (-3.0,1.3) {$t,s$};
\node[trunk] (dit) at (0,1.3) {DiT trunk};
\draw[ar] (xt)--(dit);
\draw[ar] (cond)--(dit);

\foreach \k/\x in {1/-2.55,2/-0.85,3/0.85,4/2.55}{
  \node[wnode] (w\k) at (\x,-0.1) {$w_{\k}$};
  \node[head]  (h\k) at (\x,-1.2) {$P^{(\k)}$};
  \draw[ar] (w\k) -- (h\k);
}
\draw[ar] (dit.south -| w2) -- (w2);
\draw[ar] (dit.south -| w3) -- (w3);
\draw[ar] (dit.south west) to[out=-90,in=90] (w1);
\draw[ar] (dit.south east) to[out=-90,in=90] (w4);
\node[lbl,align=center] at (-4.05,-0.1) {router\\ weights $w$};
\node[lbl,align=center] at (-4.05,-1.2) {$M{=}4$\\ factorized\\ heads};

\node[comb] (comb) at (0,-2.5)
  {$\log\Psith(\xs\!\mid\!\xt)$};
\draw[ar] (h2) -- (h2 |- comb.north);
\draw[ar] (h3) -- (h3 |- comb.north);
\draw[ar] (h1) to[out=-90,in=90] ([xshift=-16mm]comb.north);
\draw[ar] (h4) to[out=-90,in=90] ([xshift=16mm]comb.north);
\draw[ar,rounded corners=3pt] (xs.south) |- (comb.east);

\node[loss] (loss) at (0,-4) {$\mathcal{L}=-\log\Psith(\xs\!\mid\!\xt)\;+\;$ router reg};
\draw[ar] (comb)--(loss);
\end{scope}

\begin{scope}[shift={(11.6,0)}]
\node[tl,anchor=west] at (-3.2,4.4) {(b) Inference: {$J$ ancestral flow-map steps}, $t{=}0\!\to\!1$};

\node[io] (bxt) at (0,2.5) {$\xt$\ \ (start all-\Mask{}, $t{=}0$)};
\node[cnd] (bcond) at (-3.0,1.3) {$t,s$};
\node[trunk] (bdit) at (0,1.3) {DiT trunk};
\draw[ar] (bxt)--(bdit);
\draw[ar] (bcond)--(bdit);

\foreach \k/\x in {1/-2.55,2/-0.85,3/0.85,4/2.55}{
  \node[wnode] (bw\k) at (\x,-0.1) {$w_{\k}$};
  \node[head]  (bh\k) at (\x,-1.2) {$P^{(\k)}$};
  \draw[ar] (bw\k) -- (bh\k);
}
\draw[ar] (bdit.south -| bw2) -- (bw2);
\draw[ar] (bdit.south -| bw3) -- (bw3);
\draw[ar] (bdit.south west) to[out=-90,in=90] (bw1);
\draw[ar] (bdit.south east) to[out=-90,in=90] (bw4);
\node[wsel]    (bwsel) at (0.85,-0.1) {$w_{3}$};
\node[headsel] (bsel)  at (0.85,-1.2) {$P^{(3)}$};
\draw[aro] (bwsel) -- (bsel);
\node[lbl,align=center] at (-3.5,-0.1) {draw one\\ $k\sim w$};

\node[comb] (bsamp) at (0,-2.5) {sample $\xs\sim\prod_i P^{(k)}_i(\cdot\mid\xt,k)$};
\draw[aro] (bsel) -- (bsamp);

\node[io] (bnext) at (0,-4) {$\xs$};
\draw[ar] (bsamp)--(bnext);
\node[io] (bfin) at (-3.1,-4) {$x_1$ sample};
\draw[ar] (bnext) -- (bfin) node[midway,above=1pt,lbl]{at $t{=}1$};
\draw[ar] (bnext.east) to[out=0,in=0,looseness=1.5]
  node[midway,right=1pt,lbl,align=center]{$\xs\!\to\!\xt$\\ repeat $\green{J}\times$} (bxt.east);
\end{scope}
\end{tikzpicture}%
}
\caption{Illustration of training and inference in LKF with $M{=}4$ latent
components. \textbf{(a) Training.} A clean $x_1$ is corrupted to $\xt$ and
to the target $\xs$ by the two-time reveal. The Diffusion Transformer (DiT) trunk reads $\xt$ with
two-time conditioning $(t,s)$ and emits the router weights $w=(w_1,\dots,w_4)$ and the four
factorized denoisers $P^{(k)}=\prod_i P^{(k)}_i$ in one pass. The mixture log-probability of the
realized next state $\xs$ is formed by $-\log\Psith(\xs\!\mid\!\xt)$ plus the router regularizer.
\textbf{(b) Inference.} Starting from an all-\Mask{} sequence, each step runs the same trunk,
draws a single shared latent $k\sim w(\xt)$ (here $k{=}3$, highlighted), samples every masked
position from that one component $\prod_i P^{(k)}_i$, and feeds $\xs$ back as $\xt$ for the next
of {$J$} steps until $t{=}1$.}
\label{fig:lkfoverview}
\end{figure*}

The shared blocks, conditioned only on $(t,s)$, map the embedded input to a latent-agnostic
representation $h$. The router mean-pools $h$ over positions into a \emph{sequence-level} readout
of the mixture weights $w$, since ``which correlated mode is this sequence heading toward'' is a
global property of $\xt$. The latent branch tiles $h$ over the $M$ latents and runs the last $L_k$
blocks with a $k$-dependent conditioning vector, the sole channel through which $t,s$ and $k$
reach the residual stream. Hence, the token representation entering the latent blocks is identical
across $k$ (further details are in Appendix~\ref{app:arch}).

\subsection{Training objective}
\label{sec:objective}

LKF is trained by marginalizing the latent. Given a supervised next state $\xs$, write
$\ell_k=\sum_{i:\,\xt^i=\Mask}\log \Pik(\xs^i\mid\xt,k)$ for the log-probability that component $k$
assigns to the revealed content. The model scores $\xs$ with the mixture log-likelihood, summing
over the $M$ components in closed form with no sampled $k$ and no variational bound,
\begin{align}
\log\Psith(\xs\mid\xt) &=\logsumexp_k\big(\log w_k+\ell_k\big),\\
\mathcal L_{\text{nll}} &=-\E\!\left[\log\Psith(\xs\mid\xt)\right].
\label{eq:nll}
\end{align}
Since the position sum $\ell_k$ sits \emph{inside} the $\logsumexp$, the training target is the
probability that a single shared latent explains the whole masked set jointly, $\sum_k w_k\prod_i\Pik$,
the rank-$M$ correlated transition of Eq.~\eqref{eq:kernel}. Interchanging the sum and the
$\logsumexp$ would instead train the per-position marginal denoiser $\sum_k w_k\,\Pik$, which is
blind to correlation and would leave the latent without a training signal.

Because Eq.~\eqref{eq:nll} is a log-likelihood, its population
loss decomposes as a model-free entropy term plus
$\E_{\xt}\big[\mathrm{KL}\big(q(\cdot\mid\xt)\,\|\,\Psi(\cdot\mid\xt)\big)\big]$, where $q$ is the
true two-time transition, and the minimizer is therefore the pointwise KL projection of $q$ onto
the kernel class (Proposition~\ref{prop:iproj}, Appendix~\ref{proof:iproj}). Training reduces the
gap between the model and the true transition. When $q$ carries correlation that no rank-$M$
mixture can represent, this gap cannot reach zero, and Theorem~\ref{thm:mixgap} lower-bounds it by
the residual total correlation of $q$.

The supervised next state in Eq.~\eqref{eq:nll} can be any point on the path. The synthetic studies
of Appendix~\ref{app:synth} supervise the two-time transition $\xs$ directly, whereas the language
experiments supervise the clean endpoint $x_1$ at $s{=}1$.

\paragraph{Router regularizer.} A mixture router has two potential failure modes, and the
regularizer guards against both,
\begin{equation}
\mathcal R \;=\; -\,\lambda_{\text{ent}}\,\HH(\E_b[w])\;+\;\lambda_{\text{lb}}\,\E_b\big[\HH(w\mid\xt)\big],
\label{eq:reg}
\end{equation}
where $\E_b$ averages over a batch. The first term is the entropy of the batch-averaged weights.
Maximizing it, hence the minus sign, spreads usage across all $M$ components and rules out the
collapse in which the router sends every input to one component and the rest die. The second term
sets how much the latent can carry, and its sign is the opposite of the usual load-balancing choice~\citep{shazeer2017moe}.
The information the latent adds about the transition is bounded by the per-example prior entropy,
\begin{equation}
\I(\green{k};X_s\mid\xt)
=
\HH(\green{k}\mid\xt)-\HH(\green{k}\mid\xt,X_s)
\le
\HH(\green{k}\mid\xt),
\label{eq:mi_bound}
\end{equation}
With $\HH(\green{k}\mid\xt)=\HH(w\mid\xt)$, this entropy must remain
sufficiently high. We therefore set $\lambda_{\text{lb}}\le 0$ to
maximize the prior entropy and allow the exact mixture likelihood to
sharpen the posterior, driving $\HH(\green{k}\mid\xt,X_s)$ toward zero. A
positive $\lambda_{\text{lb}}$, as in conventional decisive routing,
reduces $\HH(\green{k}\mid\xt)$ and can force $\I(\green{k};X_s\mid\xt)$ toward zero,
silencing the latent.

\subsection{Inference}
\label{sec:inference}

Sampling starts from the all-\Mask{} sequence and applies {$J$} steps with
$t_j=\frac{j}{J}$ and {$t_{j+1}=\frac{j+1}{J}$}. Each step forwards the trunk once,
samples tokens at the currently masked positions from a single component $k$ under the analytic
(cached) reverse update, and keeps already-revealed tokens fixed because masking is monotone
(pseudocode in Appendix~\ref{app:pseudocode}). Any residual \Mask{} after {$J$} steps
is filled greedily. Since the latent is tied to a whole sequence, a rollout draws $k$ once at the
all-\Mask{} start, where $w$ is a constant prior, and holds it for the entire trajectory. Every
step is therefore explained by the same latent, and the decode matches the sequence-level
semantics of the training objective.

One such rollout uses a single component and therefore leaves the mixture's heterogeneity unused.
Our decode, \textbf{\textit{best-of-$M$}}, realizes that heterogeneity at inference by running all
$M$ rollouts in parallel with one component frozen in each, scoring every finished candidate with
the model's own sequence-level mixture negative ELBO (NELBO) of Eq.~\ref{eq:x0loss}, and returning
the argmin (Algorithm~\ref{alg:bestofm} in Appendix~\ref{app:pseudocode}). The model's own
likelihood is used for selection, which needs no reward model, no external judge, and no teacher,
and the $M$ rollouts are independent and therefore run concurrently at the sequential depth of a
single chain. 

The scoring passes can in fact be reduced. Every step evaluates all $M$ components in
order to marginalize over $k$, which makes the log-probability that component $k$ assigns to the
tokens revealed at that step available at no extra cost, and accumulating it along a trajectory
gives a per-component \emph{running evidence} by the time the rollout finishes. Selecting on this
quantity is free and is at least as accurate as the Monte-Carlo NELBO. Since it ranks candidates
before they finish, it also lets the weakest rollouts be pruned partway through, cutting the cost
of best-of-$M$ by roughly a factor of three.


\section{Theoretical analysis}
\label{sec:theory}

A mixture kernel~\cite{wang2026bayesian} captures the \emph{correlation}. This section makes that
quantitative in three steps: an exact decomposition of the correlation a rank-$M$ mixture carries
(Proposition~\ref{prop:tc}), a lower bound showing that a step's approximation error is at least
its \emph{uncaptured} total correlation (Theorem~\ref{thm:mixgap}), and a hardness result showing
that the mixture's nominal capacity can be structurally unusable (Theorem~\ref{thm:parity}). The
first two results are stated for the model's own step distribution and are therefore directly
measurable, and the estimators reported throughout the paper measure the quantities that
appear in the bound. Applying the total correlation of Eq.~\eqref{eq:tc_def} to the step
distribution $\Psith(\cdot\mid\xt)$ at a fixed conditioning $\xt$, write the TC of
the produced $X_s$ as
\begin{equation}
\label{eq:conditional-tc}
\TC(X_s\mid\xt)
=
\sum_{i=1}^{L}\HH(X_s^i\mid\xt)-\HH(X_s\mid\xt).
\end{equation}

\begin{proposition}[Latent decomposition and correlation capacity]
\label{prop:tc}
For the LKF transition in Eq.~\eqref{eq:kernel}, with sampled shared latent index
$\green{k\in[M]}$ and fixed $(\xt,t,s)$,
\begin{equation}
\label{eq:tc_decomp}
\TC(X_s\mid\xt)
=
\sum_{i=1}^{L}\I(\green{k};X_s^i\mid\xt)-\I(\green{k};X_s\mid\xt),
\end{equation}
and therefore
\begin{equation}
\label{eq:tc_capacity}
0\le\TC(X_s\mid\xt)
\le (L-1)\HH(\green{k}\mid\xt)
\le (L-1)\log M.
\end{equation}
\emph{Proof in Appendix~\ref{proof:tc}.}
\end{proposition}

\textbf{Approximation error from uncaptured correlation.}

\begin{theorem}[Mixture approximation lower bound]
\label{thm:mixgap}
Let $P$ be a distribution on $V^L$ and
$Q=\sum_{k=1}^{M}w_kQ_k$, where each $Q_k$ is a product distribution
and $Q>0$ on the support of $P$. With
$\gamma_Q(k\mid x)=w_kQ_k(x)/Q(x)$ and mutual information under
$\widetilde P(x,k)=P(x)\gamma_Q(k\mid x)$,
\begin{equation}
\label{eq:uncaptured-tc}
\mathrm{KL}(P\|Q)
\ge
\TC(P)-\left[\sum_{i=1}^{L}\I_{\widetilde P}(\green{k};X^i)
-\I_{\widetilde P}(\green{k};X)\right].
\end{equation}
Consequently,
\begin{equation}
\label{eq:mixgap}
\mathrm{KL}(P\|Q)\ge\big[\TC(P)-(L-1)\log M\big]_+.
\end{equation}
The bound is tight for the hidden-agreement construction of
Appendix~\ref{app:synth}. \emph{Proof in Appendix~\ref{proof:mixgap}.}
\end{theorem}

Applied per transition with $P=q(\cdot\mid\xt)$ and
$Q=\Psith(\cdot\mid\xt)$, the theorem lower-bounds the excess
population loss by uncaptured total correlation.
Although the ceiling is attainable, the next result gives a target
requiring exponentially many components.

\begin{theorem}[Parity separation]
\label{thm:parity}
Let $P_d$ be the uniform distribution on even-parity strings of $\{0,1\}^d$. If $Q$ is any
{mixture of at most $M$ product distributions} with $\mathrm{TV}(P_d,Q)\le\varepsilon<\tfrac13$, then
\begin{equation}
M\;\ge\;(1-3\varepsilon)\,2^{\,d-1}.
\end{equation}
(Proof in Appendix~\ref{proof:parity}.)
\end{theorem}

\textbf{The two theorems are falsifiable.} Because $\Ihat$ and $\TChat$ are measured on the
model's own step distribution, the two theorems make sharp predictions wherever the target total
correlation is known in closed form. Captured information must stay at zero on a target that
requires exponential mixture rank, and must rise with $M$ under the ceilings of
Eq.~\eqref{eq:tc_capacity} on a rank-$M$-representable target.

\begin{table*}[h!]
\centering
\scriptsize
\setlength{\tabcolsep}{4.5pt}
\caption{Generative perplexity (PPL) across the number of function evaluations (NFE)
on LM1B and WikiText-103 datasets, with ELBO-perplexity (ELBO-PPL), sample entropy $H$ and captured
information {$\I(\green{k})$} at $32$ NFEs. The rows marked $\dagger$ are decoded with best-of-$M$
and are budget-matched to one another with the same number of network calls during sampling.}
\label{tab:main}
\begin{tabular}{@{}c l r cccccc c c@{}}
\toprule
& & & \multicolumn{6}{c}{gen-PPL $\downarrow$ at NFE} & & \\
\cmidrule(lr){4-9}
& Method & ELBO-PPL $\downarrow$ & 1 & 2 & 4 & 8 & 16 & 32 & $H$@32 $\uparrow$ & $\I(\green{k}) \uparrow$ \\
\midrule
\multirow[c]{11}{*}[3pt]{\rotatebox[origin=c]{90}{\textbf{LM1B}}}
& MDLM        & 33.96 & 1082 & 1231 & 926 & 614 & 410 & 304 & ${\sim}4.0$ & n/a \\
& MDLM, best-of-$8^{\dagger}$ & -- & 1992 & 1401 & 770 & 428 & 281 & 199 & 5.13 & n/a \\
& SEDD        & 42.37 & 3235 & 3253 & 1117 & 485 & 303 & 225 & 4.41 & n/a \\
& SDTT teacher ($7$ rounds) & -- & 1645 & 1122 & 612 & 359 & 243 & 194 & 4.08 & n/a \\
& Di4C, best r$4$ & -- & 1375 & 899 & 453 & 236 & 156 & 115 & 3.99 & n/a \\
& ReDi, best rect$3$ & -- & \textbf{141} & \textbf{133} & \textbf{121} & \textbf{124} & \textbf{119} & 117 & 3.22 & n/a \\
& AR  & 27.61 & \multicolumn{6}{c}{$135$ \; (fixed at $128$ NFE)} & 4.34 & n/a \\
\addlinespace[2pt]
& \cellcolor{gray!12}\textbf{LKF} (\textit{\textbf{ours}}) $M{=}1$, best-of-$8^{\dagger}$ & \cellcolor{gray!12}-- & \cellcolor{gray!12}2024 & \cellcolor{gray!12}1308 & \cellcolor{gray!12}706 & \cellcolor{gray!12}400 & \cellcolor{gray!12}262 & \cellcolor{gray!12}183 & \cellcolor{gray!12}5.12 & \cellcolor{gray!12}0.00 \\
& \cellcolor{gray!12}\textbf{LKF (\textit{ours}) $M{=}1$}              & \cellcolor{gray!12}34.22 & \cellcolor{gray!12}2398 & \cellcolor{gray!12}1619 & \cellcolor{gray!12}913 & \cellcolor{gray!12}573 & \cellcolor{gray!12}429 & \cellcolor{gray!12}321 & \cellcolor{gray!12}5.22 & \cellcolor{gray!12}0.00 \\
& \cellcolor{gray!12}\textbf{LKF (\textit{ours}) $M{=}4^{\dagger}$} & \cellcolor{gray!12}35.52 & \cellcolor{gray!12}2169 & \cellcolor{gray!12}1649 & \cellcolor{gray!12}820 & \cellcolor{gray!12}407 & \cellcolor{gray!12}226 & \cellcolor{gray!12}166 & \cellcolor{gray!12}5.09 & \cellcolor{gray!12}0.86 \\
& \cellcolor{gray!12}\textbf{LKF (\textit{ours}) $M{=}8^{\dagger}$} & \cellcolor{gray!12}34.62 & \cellcolor{gray!12}1827 & \cellcolor{gray!12}1515 & \cellcolor{gray!12}625 & \cellcolor{gray!12}284 & \cellcolor{gray!12}169 & \cellcolor{gray!12}\textbf{105} & \cellcolor{gray!12}4.93 & \cellcolor{gray!12}1.48 \\
\midrule
\multirow[c]{5}{*}[1pt]{\rotatebox[origin=c]{90}{\textbf{WikiText-103}}}
& MDLM   & 28.47 & 7801 & 4344 & 1523 & 613 & 343 & 228 & 4.33 & n/a \\
& SEDD   & 29.16 & \textbf{4003} & 4052 & 1382 & 638 & 362 & 247 & 4.36 & n/a \\
\addlinespace[2pt]
& \cellcolor{gray!12}\textbf{LKF (\textit{ours}) $M{=}1$}              & \cellcolor{gray!12}\textbf{27.89} & \cellcolor{gray!12}6198 & \cellcolor{gray!12}3308 & \cellcolor{gray!12}1336 & \cellcolor{gray!12}558 & \cellcolor{gray!12}316 & \cellcolor{gray!12}223 & \cellcolor{gray!12}5.33 & \cellcolor{gray!12}0.00 \\
& \cellcolor{gray!12}\textbf{LKF (\textit{ours}) $M{=}4^{\dagger}$} & \cellcolor{gray!12}29.79 & \cellcolor{gray!12}5187 & \cellcolor{gray!12}2475 & \cellcolor{gray!12}849 & \cellcolor{gray!12}312 & \cellcolor{gray!12}164 & \cellcolor{gray!12}102 & \cellcolor{gray!12}5.28 & \cellcolor{gray!12}0.87 \\
& \cellcolor{gray!12}\textbf{LKF (\textit{ours}) $M{=}8^{\dagger}$} & \cellcolor{gray!12}28.56 & \cellcolor{gray!12}4403 & \cellcolor{gray!12}\textbf{2195} & \cellcolor{gray!12}\textbf{673} & \cellcolor{gray!12}\textbf{241} & \cellcolor{gray!12}\textbf{119} & \cellcolor{gray!12}\textbf{75} & \cellcolor{gray!12}5.25 & \cellcolor{gray!12}1.41 \\
\bottomrule
\end{tabular}

\end{table*}

\section{Experiments}
\label{sec:experiments}
\label{sec:language}

We evaluate LKF in two settings: language modeling benchmarks
LM1B~\citep{chelba2013lm1b} and WikiText-103~\citep{merity2016pointer}, and a synthetic corpora.

\subsection{Setup}
\label{sec:langprotocol}

\paragraph{Training.} LM1B uses the \texttt{bert-base-uncased} vocabulary~\citep{devlin2019bert}
($V{=}30522$) and WikiText-103 uses the GPT-2 byte-pair vocabulary~\citep{radford2019gpt2}
($V{=}50257$) at a sequence length of $128$. These
models were trained for $200\text{k}$ optimizer steps at a batch size of 512.

\paragraph{Metrics.} We report generative perplexity (gen-PPL) under a GPT-2-large
judge~\citep{radford2019gpt2} ($512$ samples, decoded and retokenized), swept over
$\mathrm{NFE}\in\{1,2,4,8,16,32\}$ and reported together with the sample entropy, since the judge
rewards low-entropy text and perplexity alone can misrank methods. We add ELBO-perplexity
(ELBO-PPL) for the likelihood models, which is the exponential of the per-token NELBO.

\subsection{Language modeling}
\label{sec:langcompare}

We validate LKF on natural language by training it at $M\in\{1,4,8\}$ against the factorized
baselines (MDLM, SEDD) and the few-step accelerators (SDTT, Di4C, ReDi) on the LM1B and
WikiText-103 datasets, together with an autoregressive transformer (AR) on LM1B.

\begin{figure}[h!]
\centering
\includegraphics[width=0.95\linewidth]{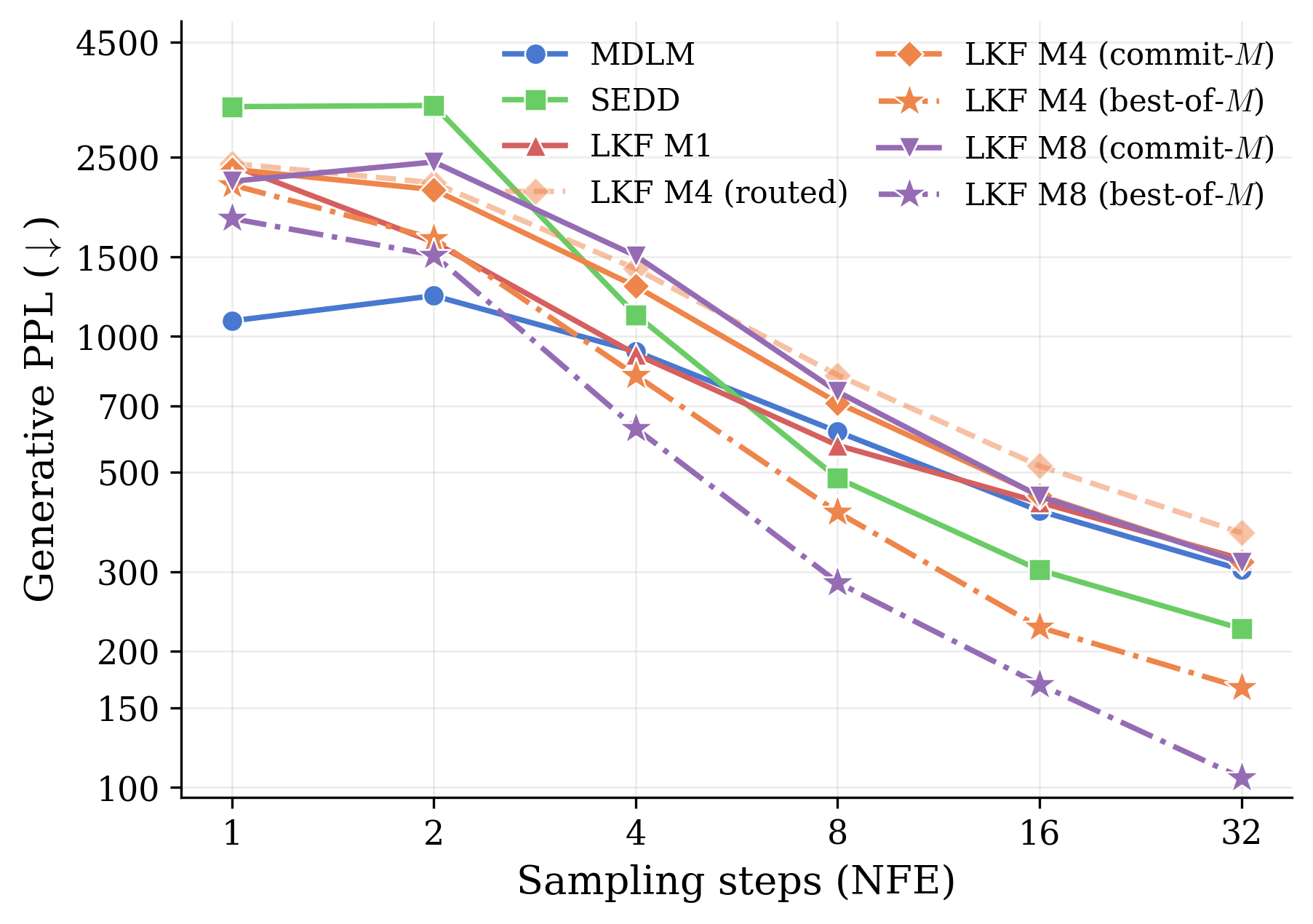}
\caption{Generative perplexity (PPL) versus number of sampling steps on LM1B, comparing LKF at
$M{=}\{1,4,8\}$ against the baseline methods.}
\label{fig:track2}
\end{figure}

\begin{figure*}[h!]
\centering
\includegraphics[width=0.8\textwidth]{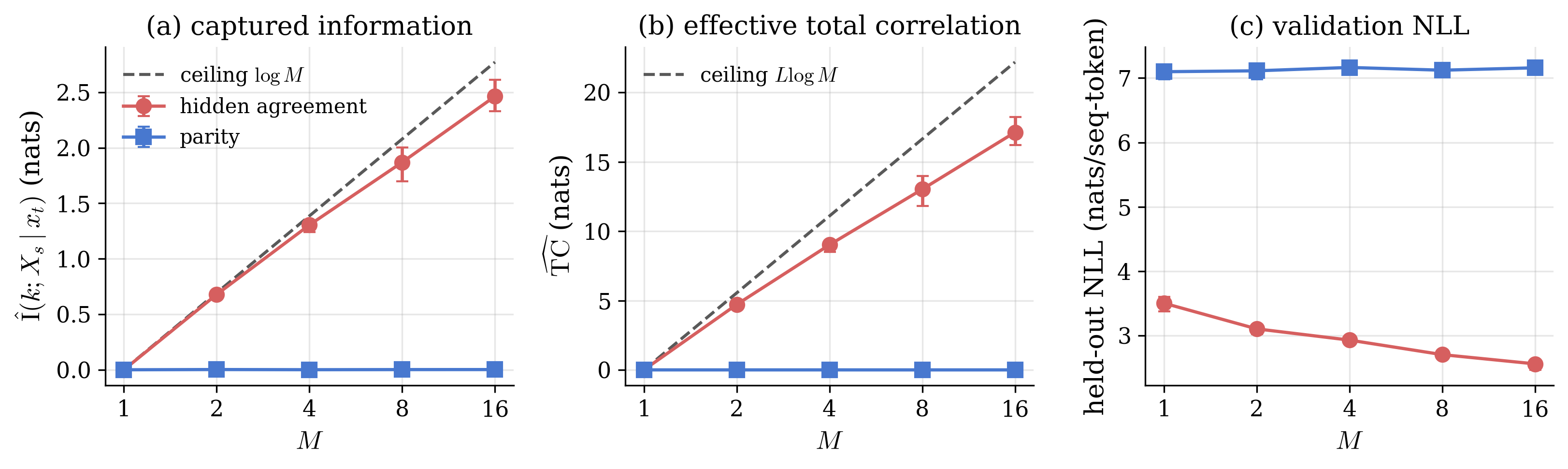}
\caption{Illustration of correlation and information captured by LKF on synthetic corpora. On \textit{hidden agreement}, the captured information $\Ihat$ and the
effective total correlation $\TChat$ rise monotonically with $M$, tracking their ceilings ($\log M$
{and $(L-1)\log M$, dashed)}, and held-out NLL falls accordingly. On \textit{parity}, every captured-correlation metric stays at zero for all $M$.}
\label{fig:f3}
\end{figure*}

\subsubsection{LM1B}
\label{sec:mainresults}

Adding mixture components improves few-step generation on LM1B, and the improvement grows with the
number of components, as shown in Table~\ref{tab:main} and Figure~\ref{fig:track2}. The
eight-component model reaches the lowest (best) generative perplexity of $105$ at $32$ evaluations,
and it does so with a higher sample entropy than the baselines (sample texts in
Appendix~\ref{app:samples}). The gains also come from how the sampling budget is spent rather than
from spending more of it.


We also retrained an autoregressive model (a causal transformer) under similar training conditions.
Its full-context factorization gives the lowest clean likelihood of any model, with validation
perplexity $27.61$ against MDLM's $33.96$ and our $M{=}1$ base at $34.22$, but its generative
perplexity is $135$ at a fixed $128$ function evaluations, one per token, and it therefore cannot
enter the few-step regime. The two are not directly comparable, since best-of-$8$ spends $M$
rollouts to reach $105$ but has a sequential depth of $32$ steps rather than AR's $128$, which
makes it cheaper in latency and more expensive in total evaluations. LKF therefore wins on judged
sample quality wherever the budget can be spent in parallel.

\subsubsection{WikiText-103}
\label{sec:wt103}

We repeat the evaluation on WikiText-103 by only changing the corpus and its GPT-2 vocabulary.
The single-component LKF again matches the retrained MDLM and edges ahead of both factorized
baselines with the lowest ELBO-perplexity.
The mixture rows replicate and strengthen the LM1B pattern on this second corpus. At $M{=}4$, the
mixture is live at $\I(k){=}0.87$ nats and reaches $102$ at $32$ evaluations, more than $2\times$
below both baselines, and widening to $M{=}8$ continues the trend on both axes, raising the captured
information to $\I(k){=}1.41$ nats and taking the corpus to $75$, which is $3.0\times$ below MDLM
and $3.3\times$ below SEDD.

The likelihood cost of the mixture is not monotone in $M$. Moving from one component to four trades
ELBO-perplexity $27.89$ for $29.79$, but the eight-component model recovers most of that to $28.56$,
and the wider mixture therefore buys its generation gain almost for free in likelihood terms. The
two factorized models have unigram entropies of $4.33$ (MDLM) and $4.36$ (SEDD) at $32$
evaluations, well below LKF ($5.25$). 


\subsubsection{Comparison with few-step accelerators}
\label{sec:redidi4c}


We run two few-step samplers from a trained teacher, namely ReDi from the DUO
base~\citep{sahoo2025duo} and Di4C from a seven-round SDTT student (the round-by-round
results are in Appendix~\ref{app:anchordetails} Tables~\ref{tab:di4crounds}
and~\ref{tab:redirounds}). The two behave differently from one another. ReDi holds the
single-evaluation end with a nearly flat curve, but it earns its per-round gains through a steady
entropy collapse that the GPT-2 judge rewards as fidelity.
Both models stop improving around their fifth round, and because generative perplexity alone
misranks them, we also compute their entropy.

As shown in Table~\ref{tab:main}, the distilled students remain ahead in the ultra-few-step regime,
since they inherit a converged teacher's calibration where the from-scratch mixture's
cold-start low-NFE points are weakest. The ordering reverses beyond roughly eight evaluations.
LKF overtakes the SDTT teacher and, at $32$ evaluations, reaches the lowest perplexity of any row
($105$), passing both the distilled and the rectified accelerators without using any teacher.
ReDi's low few-step numbers are not directly comparable, since they carry the entropy-collapse
failure mode noted above.

\subsection{Synthetic corpora with known correlation}
\label{sec:synthexp}

We built two synthetic corpora on which the TC is known in closed form and trained LKF on them at
$M\in\{1,2,4,8,16\}$. The first is
\textit{hidden agreement}, whose joint is exactly representable by a rank-$M$ mixture, and the
second is \textit{parity}, whose representation cost is exponential in the sequence length. The two
theorems make opposite predictions on these corpora, and both predictions hold
(Figure~\ref{fig:f3}). On hidden agreement, the mixture uses the latent, and uses it more as $M$
grows, where the captured information rises from $0$ to $2.46$ nats ($89\%$ of the $\log 16$ ceiling)
and the effective total correlation from $0$ to $17.1$ nats, both under their ceilings $\log M$ and
{$(L-1)\log M$, while held-out NLL falls} from $3.50$ to $2.56$ nats. On parity, the
latent stays off at every $M$, with captured information and total correlation pinned at zero and
the NLL flat, as the
hardness result of Theorem~\ref{thm:parity} predicts. The full experiments on the captured correlation and the router-mode analysis are in
Appendix~\ref{app:synth}.

\subsection{Ablations}
\label{sec:ablations}

We freeze the trained LKF model and vary how the latent is selected and used at inference (full details in Appendix~\ref{app:decode}).

\paragraph{The gain is a mixture property, and it grows with $M$.} We used the same selector and the judge to isolate the effect of best-of-$M$ policy, and varied only where the candidates come from. Replacing the $M$
components by $M$ noise draws of one component costs $16\%$ at $M{=}4$ ($157$ against $186$) and
$35\%$ at $M{=}8$ ($106$ against $163$). \bon{Eight i.i.d. chains reranked by their own likelihood
at the identical budget were also given to the trained MDLM baseline and to our
architecture-matched $M{=}1$ model. The reranking helps them substantially, taking them from $363$
to $199$ and from $327$ to $183$ at $32$ steps, but the mixture still reaches $106$ at the same
$8(J{+}T)$ calls, which is a $47\%$ and a $42\%$ advantage that widens monotonically from $14\%$ at
four steps.} The components therefore provide a candidate diversity, and
that advantage grows as the mixture grows (Appendix~\ref{app:decode}).

\paragraph{The selector is native, and the cheap version is the better one.} Our experiments show
that the ranking of the components matters, since a random pick from the same pool gives $327$ at
$M{=}8$, while every scored pick lands between $106$ and $163$. Selection here uses only the
model's own sequence likelihood, with no reward model, no external judge, and no temperature
tuning. The \textit{Monte-Carlo mixture NELBO} can be replaced with the \textit{running evidence}
that the trunk already produces during generation, which reaches $125$ against $131$ at $M{=}8$,
while removing the $MT$ scoring passes. The same ordering holds at every step, where the running evidence matches or beats the Monte-Carlo selector at a strictly lower
cost. We nonetheless report Table~\ref{tab:main} with the Monte-Carlo selector, which is the
conservative choice. The candidates can also be pruned mid-trajectory, because the evidence is
available at every step. \textit{Successive halving} reaches $145$ for $96$ calls at $M{=}8$,
which is three times cheaper than full best-of-$8$ for an $11\%$ higher perplexity.


\paragraph{Limitations}
\label{sec:discussion}
We note two main limitations of LKF. The first is that the single-step exactness is capped by the
per-step factorization. Each step is factorized given $(\xt,k)$, and one disagreeing token can
therefore push a sample off support independently of $M$, which only smaller steps or a richer
per-step structure would remove. The second is that the learned router is not very informative
along sampling trajectories. Making the router track component quality would turn best-of-$M$ into
a cheaper single-rollout policy.

\section{Conclusion}
\label{sec:conclusion}

We presented LKF, a discrete flow map whose finite-time kernel is a mixture of $M$ factorized
components, which expresses correlated steps natively and without a
teacher. We provided the theory to support it and conducted experiments on the LM1B and WikiText-103 datasets that give the best
few-step generative perplexity among the baseline methods.

\bibliography{references}

\onecolumn
\appendix
\section{Notation}
\label{app:notation}

Table~\ref{tab:notation} collects the symbols used throughout the paper.

\begin{table}[h!]
\centering\small
\caption{Notations used throughout the paper.}
\label{tab:notation}
\begin{tabular}{@{}ll@{}}
\toprule
Symbol & Meaning \\
\midrule
$L$ & sequence length (number of positions) \\
$V$, $\mathcal V=\{1,\dots,V\}$ & vocabulary size and vocabulary \\
$\Mask$ & absorbing (mask) symbol, $\Mask\notin\mathcal V$ \\
$x=(x^1,\dots,x^L)$ & a clean sequence, $x^i$ the token at position $i$ \\
$p_{\text{data}},\,p_{\text{model}}$ & data and model distributions over clean sequences \\
$t,s\in[0,1]$ & two times with $t<s$ (noisier to cleaner) \\
$\kappa_t$ & keep-probability schedule ($\kappa_0{=}0$, $\kappa_1{=}1$) \\
$\xt$ & partially masked sequence at time $t$ \\
$\Psith(\xs\mid\xt)$ & finite-time flow-map kernel from $t$ to $s$ \\
$M$ & number of mixture components \\
\green{$k\in[M]$} & \green{sampled latent component index; also used as the mixture-component index} \\
$w_k(\xt,t,s)$ & router weight of component $k$, with $\sum_k w_k=1$ \\
$\Pik(\cdot\mid\xt,k,t,s)$ & component $k$'s token distribution at position $i$ \\
$J$ & number of sampling steps, $\mathrm{NFE}=J$ \\
$B,\,D,\,L_k$ & batch size, model width, number of late (latent) blocks \\
$\HH,\,\I,\,\TC$ & entropy, mutual information, total correlation \\
\bottomrule
\end{tabular}
\end{table}

\section{The MDLM-equivalent clean-posterior objective}
\label{app:objx0}

The language experiments supervise the clean endpoint $x_1$ at $s{=}1$, which turns the LKF
objective of Eq.~\eqref{eq:nll} into a clean-posterior mixture whose $M{=}1$ case coincides with
MDLM. We build this instantiation around two requirements. The first is that its $M{=}1$ instance
recovers MDLM's absorbing evidence lower bound, which lets the mixture effect be read against a
matched anchor. The second is that its $M>1$ instance trains the latent on whole masked sets, which
is what makes $k$ carry sequence-level correlation. Following MDLM, the corruption uses the
loglinear schedule
$\sigma(\tau)=-\log\!\big(1-(1-\varepsilon)\tau\big)$ with ELBO weight
$\omega(\tau)=\mathrm d\sigma/\mathrm{expm1}(\sigma)=1/\tau$. For each step, we draw a single time
$\tau\sim U[\varepsilon,1]$ with antithetic (stratified) sampling, mask each position with
probability $(1-\varepsilon)\tau$ to form $x_\tau$, and query the clean head ($s{=}1$) under the
SUBS parameterization, which zeros the mask logit ($\text{logit}_{\Mask}\!\to\!-\infty$) and copies
revealed tokens through. Writing $\ell_i^k=\log P_i^k(x_1^i\mid x_\tau)$ for the clean-token
log-probability of component $k$ at position $i$ and $\mathcal M$ for the set of masked,
non-padding positions, the per-example loss is the \emph{sequence-level} mixture likelihood
\begin{equation}
\mathcal L_{x_0}\;=\;-\,\omega(\tau)\,
\logsumexp_k\!\Big(\log w_k+\textstyle\sum_{i\in\mathcal M}\ell_i^k\Big),
\label{eq:x0loss}
\end{equation}
accumulated over the batch and normalized by the total count of non-padding tokens, mirroring
MDLM's attention-mask normalization. As in the main objective, keeping the position-sum inside the
$\logsumexp$ makes the target $\sum_k w_k \prod_{i\in\mathcal M} P_i^k(x_1^i\mid x_\tau)$, the
rank-$M$ clean-endpoint posterior of Eq.~\eqref{eq:kernel}. 
{Its population objective is a positively weighted conditional cross-entropy, equivalently a
positively weighted expected conditional KL divergence up to a parameter-independent entropy term.
The training loop is Algorithm~\ref{alg:train}.}

\begin{proposition}[Clean-posterior KL decomposition]
\label{prop:x0proj}
Let $q(\cdot\mid x_\tau)$ be the true posterior over masked clean tokens and let
$\Psi^{x_0}_\theta(\cdot\mid x_\tau)$ be the LKF clean-endpoint mixture. With the same
per-example token normalization as Eq.~\eqref{eq:x0loss}, the population objective is
\[
\E\!\left[\frac{\omega(\tau)}{N_{\rm tok}}
\left\{\HH\big(q(\cdot\mid x_\tau)\big)+
\mathrm{KL}\big(q(\cdot\mid x_\tau)\\|\Psi^{x_0}_\theta(\cdot\mid x_\tau)\big)\right\}\right].
\]
Hence, any population minimizer minimizes the corresponding weighted expected conditional KL.
If the model class can realize the conditional projections jointly, it realizes them almost surely.
Theorem~\ref{thm:mixgap} lower-bounds each conditional KL by uncaptured total correlation.
\emph{Proof in Appendix~\ref{proof:x0proj}.}
\end{proposition}

For likelihood \emph{evaluation}, we additionally report the standard object that MDLM and SEDD
report, which is the Rao-Blackwellized NELBO of the marginal denoiser.

\begin{proposition}[Marginal NELBO for evaluation]
\label{prop:melbo}
The marginal denoiser $\bar x_\theta(x_\tau)_i=\sum_k w_kP_i^k(\cdot\mid x_\tau)$
satisfies the SUBS constraints. After multiplying the Monte Carlo estimate by
$1-\varepsilon$, its weighted per-position loss is the truncated Rao--Blackwellized absorbing
NELBO evaluated at $\bar x_\theta$. It upper-bounds the negative log-likelihood of the
continuous-time reverse process defined by this marginal denoiser and is comparable across $M$.
\emph{Proof in Appendix~\ref{proof:melbo}.}
\end{proposition}

The containment of MDLM at $M{=}1$ is exact at both the objective and the sampler level. Hence, the
$M{=}1$ language cell is an anchor by construction rather than by tuning.

\begin{proposition}[Exact MDLM containment at $M{=}1$]
\label{prop:contain}
At $M{=}1$, Eq.~\eqref{eq:x0loss} reduces to MDLM's weighted SUBS objective, and the LKF
analytic update equals MDLM's cached ancestral update on every grid
$t_j=j/J$, $t_{j+1}=(j+1)/J$. Thus, the $M{=}1$ objective and sampler coincide with MDLM.
\emph{Proof in Appendix~\ref{proof:contain}.}
\end{proposition}

\begin{remark}[Two likelihood objects at $M>1$]
\label{rem:twolik}
The sequence-level mixture objective and the marginal NELBO are distinct. The former rewards one
latent component for explaining the whole masked set, whereas the latter depends only on
per-position marginals. Thus, the marginal ELBO-perplexity does not measure the joint correlation
that a frozen-component finite-step sampler exploits.
\end{remark}

\section{Training and inference}
\label{app:pseudocode}

We present the training algorithm for LKF with the
clean-posterior mixture objective in Algorithm~\ref{alg:train}, and the inference with the best-of-$M$
decode policy in Algorithm~\ref{alg:bestofm}.

\begin{algorithm}[h!]
\caption{LKF training (exact-mixture likelihood of the clean tokens)}
\label{alg:train}
\begin{algorithmic}[1]
\REQUIRE corpus $\mathcal D$; network $\theta$; mixture size $M$; training steps $T$
\FOR{$n=1$ \TO $T$}
  \STATE sample a clean batch $x_1\sim\mathcal D$ and one noise level $\tau\sim U[\varepsilon,1]$ per sequence
  \STATE $x_\tau\leftarrow$ mask each position of $x_1$ with probability $(1-\varepsilon)\tau$; \ $\mathcal M\leftarrow$ masked positions
  \STATE run the network on $(x_\tau,\,s{=}1)$: component heads $P^1,\dots,P^M$ and router $w$
  \STATE $\ell_k\leftarrow\sum_{i\in\mathcal M}\log P_i^k(x_1^i)$ for each $k$ \COMMENT{how well component $k$ explains the \emph{whole} masked set}
  \STATE $\log\Psi\leftarrow\logsumexp_k\big(\log w_k+\ell_k\big)$ \COMMENT{exact marginalization over $k$; no sampled latent}
  \STATE $\mathcal L\leftarrow\frac1\tau\big({-}\log\Psi\big)$ averaged over the batch $+$ router regularizer $\mathcal R$ \COMMENT{ELBO weight; Eq.~\eqref{eq:reg}}
  \STATE gradient step on $\mathcal L$
\ENDFOR
\RETURN $\theta$
\end{algorithmic}
\end{algorithm}

\begin{algorithm}[h!]
\caption{LKF inference (best-of-$M$ decode at {$J$ steps})}
\label{alg:bestofm}
\begin{algorithmic}[1]
\REQUIRE trained network $\theta$; mixture size $M$; sampling steps {$J$}; scoring draws $R$
\FORALL{components $k=1,\dots,M$, one rollout each, in parallel}
  \STATE $x\leftarrow$ all-\Mask{}
  \FOR{{$j=0$ \TO $J-1$}}
    \STATE run the network on $(x,\,s{=}1)$ and read off component $k$'s clean prediction $P^k$
    \STATE reveal each \Mask{} position with probability {$\tfrac{1}{J-j}$}, drawing its token from $P_i^k$ \COMMENT{scheduled reveal rate; revealed tokens stay fixed}
  \ENDFOR
  \STATE $\hat x^k\leftarrow x$, with any residual \Mask{} filled by $\arg\max P_i^k$
\ENDFOR
\STATE $S_k\leftarrow$ mixture NELBO of $\hat x^k$ (Eq.~\eqref{eq:x0loss}), averaged over $R$ draws of $(\tau,\text{mask})$ shared across $k$ \COMMENT{the model scores its own candidates; no external judge}
\RETURN $\hat x^{k^\star}$ with $k^\star=\arg\min_k S_k$
\end{algorithmic}
\end{algorithm}

\section{Proofs}
\label{app:proofs}

Throughout, all state spaces are finite, all entropies and divergences are in nats, and all
quantities conditioned on $\xt$ are understood pointwise in $\xt$ unless an outer expectation is
written explicitly.

\subsection{Rate mixtures remain factorized}
\label{proof:rates}

Call a CTMC generator $R$ \emph{coordinatewise} if it decomposes as $R=\sum_{i=1}^L R^i$, where
$R^i$ assigns nonzero rate only to transitions that change coordinate $i$ alone. All factorized
discrete diffusions and flows have coordinatewise generators.

\begin{proposition}[Locally averaged rate mixtures remain coordinatewise]
\label{prop:rates}
Let $R_1,\dots,R_M$ be coordinatewise generators on a finite state space, with uniformly bounded
rates on the time interval considered, and let
$\bar R(x,y,t)=\sum_k w_k(x,t)R_k(x,y,t)$. Then $\bar R$ is coordinatewise, and its
length-$\delta$ transition kernel differs from the product of its coordinate marginals by
$O(\delta^2)$. This concerns a locally averaged generator, not a trajectory generated by one
globally frozen component.
\end{proposition}

\emph{Proof.}

Write $R_k=\sum_{i=1}^L R_k^i$, where $R_k^i(x,y)=0$
unless $y$ differs from $x$ in coordinate $i$ alone. Fix arbitrary measurable weights
$w_k(x,t)\ge 0$ with $\sum_k w_k(x,t)=1$ and set
$\bar R(x,y,t)=\sum_k w_k(x,t)\,R_k(x,y,t)$ for $y\neq x$, with the diagonal defined by row sums.
The off-diagonal entries of $\bar R$ are nonnegative as convex combinations of nonnegative rates,
and its rows sum to zero by construction, so $\bar R$ is a valid CTMC generator.

We first verify the coordinatewise structure. If $y$ differs from $x$ in two or more coordinates, then
$R_k(x,y,t)=0$ for every $k$, because each $R_k$ is coordinatewise, hence
$\bar R(x,y,t)=\sum_k w_k(x,t)\cdot 0=0$. Defining
$\bar R^i(x,y,t)=\sum_k w_k(x,t)\,R_k^i(x,y,t)$ gives $\bar R=\sum_i\bar R^i$ with each $\bar R^i$
supported on single-coordinate-$i$ changes, so $\bar R$ is coordinatewise.

The sample paths of a CTMC are piecewise constant with
jumps at isolated times, and each jump moves the state from some $x$ to some $y$ with
$\bar R(x,y,t)>0$. By the coordinatewise structure, every such $y$ differs from $x$ in exactly one coordinate, so each
jump changes exactly one coordinate. Since two jumps coincide in time with probability zero, the
process changes one coordinate at a time almost surely.

We now turn to the infinitesimal factorization. Rates are bounded on the finite state space
and any compact time interval, say by $\Lambda$. The number of jumps in a window of length
$\varepsilon$ is stochastically dominated by a Poisson variable $N$ with mean
$\Lambda\varepsilon$, and
$\Pr[N\ge 2]\le\tfrac12(\Lambda\varepsilon)^2\le(\Lambda\varepsilon)^2$. Hence, up to an
$O(\varepsilon^2)$ total-variation error, we may assume the window contains zero jumps or one jump.

Writing
$\mu=P_\varepsilon(x,\cdot)$ for the time-$\varepsilon$ kernel started at $x$, condition on the
number of jumps. With zero jumps, the state stays at $x$. With exactly one jump, the single change
is in some coordinate $i$ by the one-coordinate-at-a-time property, an event of probability
$a_i=O(\varepsilon)$, after which the state lies in $\{y:y^j=x^j\ \forall j\neq i\}$, and we call
its conditional law $\mu_i$. Together with the jump-count bound,
\begin{equation*}
\mu=\Big(1-\sum_i a_i\Big)\,\delta_x+\sum_{i=1}^L a_i\,\mu_i+O(\varepsilon^2).
\end{equation*}

Next, project this decomposition onto
coordinate $i$. The point mass $\delta_x$ projects to $\delta_{x^i}$, each $\mu_j$ with $j\neq i$
leaves coordinate $i$ at $x^i$ and also projects to $\delta_{x^i}$, and $\mu_i$ projects to its
own $i$-th marginal $\nu_i$. Hence, the $i$-th marginal of $\mu$ is
\begin{equation*}
\Big(1-\sum_j a_j\Big)\delta_{x^i}+\sum_{j\neq i}a_j\,\delta_{x^i}+a_i\,\nu_i+O(\varepsilon^2)
=(1-a_i)\,\delta_{x^i}+a_i\,\nu_i+O(\varepsilon^2).
\end{equation*}

Finally, expand the product of these $L$ marginals. The all-$\delta$ term contributes
$\prod_i(1-a_i)\,\delta_x=\big(1-\sum_i a_i\big)\delta_x+O(\varepsilon^2)$, since every
higher-order term of the expansion of $\prod_i(1-a_i)$ carries at least two factors
$a_ia_j=O(\varepsilon^2)$. Each term selecting the $\nu_i$ factor exactly once contributes
$a_i\prod_{j\neq i}(1-a_j)\,\big(\delta_{x^1}\otimes\cdots\otimes\nu_i\otimes\cdots\otimes
\delta_{x^L}\big)=a_i\,\mu_i+O(\varepsilon^2)$, using $\prod_{j\neq i}(1-a_j)=1-O(\varepsilon)$
and the fact that $\mu_i$ agrees with $x$ off coordinate $i$. Every term selecting two or more
$\nu$ factors carries $a_ia_j=O(\varepsilon^2)$. Therefore
\begin{equation*}
\prod_{i=1}^L\big[(1-a_i)\delta_{x^i}+a_i\nu_i\big]
=\Big(1-\sum_i a_i\Big)\delta_x+\sum_i a_i\,\mu_i+O(\varepsilon^2)=\mu+O(\varepsilon^2),
\end{equation*}
which is the claimed $O(\varepsilon^2)$ total-variation bound. {The $\Theta(1/J)$ discretization
floor for simulating such locally averaged processes with $J$ product steps} is Theorem~1 of
\citet{hayakawa2024di4c} and is not re-proved here. \qed

\subsection{The LKF objective trains the KL projection}
\label{proof:iproj}

\begin{proposition}[Population KL decomposition]
\label{prop:iproj}
Fix $(t,s)$ and a class $\mathcal C$ of transition kernels. Then
\[
\mathcal L_{\mathrm{nll}}(\Psi)=
\E_{\xt}\HH\big(q(\cdot\mid\xt)\big)+
\E_{\xt}\mathrm{KL}\big(q(\cdot\mid\xt)\\|\Psi(\cdot\mid\xt)\big).
\]
Thus, any minimizer over $\mathcal C$ minimizes the expected conditional KL, and the excess over
the unrestricted optimum is that expected KL. A pointwise conclusion additionally requires
$\mathcal C$ to permit the conditional projections to be chosen jointly.
\end{proposition}

\emph{Proof.}

Fix $\xt$ and abbreviate $q=q(\cdot\mid\xt)$ and
$\Psi=\Psi(\cdot\mid\xt)$. Adding and subtracting $\log q(\xs)$ inside the expectation,
\begin{equation*}
-\E_{\xs\sim q}\big[\log\Psi(\xs)\big]
=-\sum_{\xs}q(\xs)\log q(\xs)+\sum_{\xs}q(\xs)\log\frac{q(\xs)}{\Psi(\xs)}
=\HH(q)+\mathrm{KL}\big(q\,\|\,\Psi\big),
\end{equation*}
where the first sum is the Shannon entropy of $q$, and the second is the KL divergence by
definition.

The population loss draws $\xt$ from the true
intermediate law and $\xs$ from $q(\cdot\mid\xt)$, so averaging this identity over $\xt$
gives
$\mathcal L_{\mathrm{nll}}(\Psi)=\E_{\xt}\big[\HH\big(q(\cdot\mid\xt)\big)\big]
+\E_{\xt}\big[\mathrm{KL}\big(q(\cdot\mid\xt)\,\|\,\Psi(\cdot\mid\xt)\big)\big]$, the stated
decomposition.

The entropy term is independent of $\Psi$, so minimizing over $\mathcal C$ is equivalent to
minimizing the expected conditional KL. If $\mathcal C$ is rectangular across conditioning
states, or contains a kernel formed from the pointwise conditional projections, the minimization
separates pointwise. Otherwise, only the expected-KL statement follows. The unrestricted optimum is
$q$, for which the KL is zero, so the excess over that optimum is the displayed expected KL.
\qed

\subsection{Proof of Proposition~\ref{prop:tc}
(latent decomposition and correlation capacity)}
\label{proof:tc}

\begin{snugshade}\noindent
\textbf{Proposition~\ref{prop:tc}}
(latent decomposition and correlation capacity). \itshape
Let $\green{k\in[M]}$ denote the sampled shared latent component index of the LKF transition.
Then
\[
\TC(X_s\mid\xt)
=
\sum_{i=1}^{L}\I(\green{k};X_s^i\mid\xt)
-
\I(\green{k};X_s\mid\xt),
\]
and
\[
0
\le
\TC(X_s\mid\xt)
\le
(L-1)\HH(\green{k}\mid\xt)
\le
(L-1)\log M.
\]
\end{snugshade}

\emph{Proof.}
Fix $\xt$ and $(t,s)$ and suppress them temporarily from the notation.
Conditioned on $\green{k}$, the LKF component is a product distribution, so
\begin{equation}
\label{eq:prop1-conditional-entropy}
\HH(X_s\mid \green{k})
=
\sum_{i=1}^{L}\HH(X_s^i\mid \green{k}).
\end{equation}
Therefore,
\begin{align}
\sum_{i=1}^{L}\I(\green{k};X_s^i)-\I(\green{k};X_s)
&=
\sum_{i=1}^{L}\left[\HH(X_s^i)-\HH(X_s^i\mid \green{k})\right]
-
\left[\HH(X_s)-\HH(X_s\mid \green{k})\right]
\nonumber\\
&=
\sum_{i=1}^{L}\HH(X_s^i)-\HH(X_s)
\nonumber\\
&=
\TC(X_s),
\label{eq:prop1-decomposition-proof}
\end{align}
where Eq.~\eqref{eq:prop1-conditional-entropy} cancels the
conditional-entropy terms. Restoring the conditioning on $\xt$ proves
Eq.~\eqref{eq:tc_decomp}.

For every coordinate, $X_s^i$ is a deterministic function of $X_s$,
so the data-processing inequality gives
\[
\I(\green{k};X_s^i\mid\xt)
\le
\I(\green{k};X_s\mid\xt).
\]
Choose
\[
j\in\operatorname*{arg\,max}_{1\le i\le L}\I(\green{k};X_s^i\mid\xt).
\]
Using Eq.~\eqref{eq:tc_decomp},
\begin{align}
\TC(X_s\mid\xt)
&\le
\sum_{i=1}^{L}\I(\green{k};X_s^i\mid\xt)
-
\I(\green{k};X_s^j\mid\xt)
\nonumber\\
&=
\sum_{i\ne j}\I(\green{k};X_s^i\mid\xt)
\nonumber\\
&\le
(L-1)\HH(\green{k}\mid\xt)
\nonumber\\
&\le
(L-1)\log M.
\label{eq:prop1-capacity-proof}
\end{align}
The final inequality follows because $\green{k}$ takes at most $M$ values.
Finally,
\[
\TC(X_s\mid\xt)
=
\mathrm{KL}\!\left(
\Pr(X_s\mid\xt)
\middle\|
\prod_{i=1}^{L}\Pr(X_s^i\mid\xt)
\right)
\ge 0.
\]
This proves Eq.~\eqref{eq:tc_capacity}.
\qed

\subsection{Proof of Theorem~\ref{thm:mixgap}
(mixture approximation lower bound)}
\label{proof:mixgap}

\begin{snugshade}\noindent
\textbf{Theorem~\ref{thm:mixgap}}
(mixture approximation lower bound). \itshape
With $\gamma_Q(k\mid x)=w_kQ_k(x)/Q(x)$ and
$\widetilde P(x,k)=P(x)\gamma_Q(k\mid x)$,
\[
\mathrm{KL}(P\|Q)
\ge
\TC(P)
-
\left[
\sum_i\I_{\widetilde P}(\green{k};X^i)
-
\I_{\widetilde P}(\green{k};X)
\right]
\ge
\left[\TC(P)-(L-1)\log M\right]_+.
\]
\end{snugshade}

\emph{Proof.}
Define
\begin{align}
\gamma_Q(k\mid x)&=\frac{w_kQ_k(x)}{Q(x)},
&\widetilde P(x,k)&=P(x)\gamma_Q(k\mid x),\nonumber\\
\alpha_k&=\green{\widetilde P(k)},
&\widetilde P_k&=\green{\widetilde P(X\mid k)},
\label{eq:lifted-definitions}
\end{align}
and let
\begin{equation}
\label{eq:lifted-model}
\widetilde Q(x,k)=w_kQ_k(x).
\end{equation}
Then
\[
\widetilde P(k\mid x)
=
\widetilde Q(k\mid x)
=
\gamma_Q(k\mid x).
\]
Applying the KL chain rule by conditioning first on $X$ gives
\begin{equation}
\label{eq:kl-chain-x}
\mathrm{KL}(\widetilde P\|\widetilde Q)
=
\mathrm{KL}(P\|Q).
\end{equation}
Conditioning instead on $\green{k}$ gives
\begin{equation}
\label{eq:klchain}
\mathrm{KL}(\widetilde P\|\widetilde Q)
=
\mathrm{KL}(\alpha\|w)
+
\sum_{k:\alpha_k>0}\alpha_k
\mathrm{KL}(\widetilde P_k\|Q_k).
\end{equation}
Together with Eq.~\eqref{eq:kl-chain-x}, Eq.~\eqref{eq:klchain} gives the KL decomposition.

Let $\widetilde P_k^i$ denote the $i$-th marginal of
$\widetilde P_k$. Since $Q_k$ is a product distribution,
\begin{align}
\mathrm{KL}(\widetilde P_k\|Q_k)
&=
\mathrm{KL}\!\left(
\widetilde P_k
\middle\|
\prod_{i=1}^{L}\widetilde P_k^i
\right)
+
\sum_{i=1}^{L}\mathrm{KL}(\widetilde P_k^i\|Q_k^i)
\nonumber\\
&=
\TC(\widetilde P_k)
+
\sum_{i=1}^{L}\mathrm{KL}(\widetilde P_k^i\|Q_k^i)
\nonumber\\
&\ge
\TC(\widetilde P_k).
\label{eq:slice-tc-bound}
\end{align}
Consequently,
\begin{equation}
\label{eq:weighted-slice-bound}
\mathrm{KL}(P\|Q)
\ge
\mathrm{KL}(\alpha\|w)
+
\sum_{k:\alpha_k>0}\alpha_k\TC(\widetilde P_k).
\end{equation}

All remaining entropies and mutual informations are evaluated under
the lifted target joint $\widetilde P$. Since
$\green{\widetilde P(X\mid k)=\widetilde P_k}$,
\begin{align}
\sum_{k:\alpha_k>0}\alpha_k\TC(\widetilde P_k)
&=
\sum_{i=1}^{L}\HH_{\widetilde P}(X^i\mid \green{k})
-
\HH_{\widetilde P}(X\mid \green{k})
\nonumber\\
&=
\TC(P)
-
\left[
\sum_{i=1}^{L}\I_{\widetilde P}(\green{k};X^i)
-
\I_{\widetilde P}(\green{k};X)
\right].
\label{eq:slice-tc-identity}
\end{align}
Substituting Eq.~\eqref{eq:slice-tc-identity} into
Eq.~\eqref{eq:weighted-slice-bound} and dropping the nonnegative term
$\mathrm{KL}(\alpha\|w)$ proves Eq.~\eqref{eq:uncaptured-tc}.

To obtain the cardinality bound, let
\[
j\in\operatorname*{arg\,max}_{1\le i\le L}\I_{\widetilde P}(\green{k};X^i).
\]
Since $X^j$ is a deterministic projection of $X$,
\[
\I_{\widetilde P}(\green{k};X)
\ge
\I_{\widetilde P}(\green{k};X^j).
\]
Hence,
\begin{align}
\sum_{i=1}^{L}\I_{\widetilde P}(\green{k};X^i)
-
\I_{\widetilde P}(\green{k};X)
&\le
\sum_{i\ne j}\I_{\widetilde P}(\green{k};X^i)
\nonumber\\
&\le
(L-1)\HH_{\widetilde P}(\green{k})
\nonumber\\
&\le
(L-1)\log M.
\label{eq:captured-cardinality-bound}
\end{align}
Combining this with Eq.~\eqref{eq:uncaptured-tc} and
$\mathrm{KL}(P\|Q)\ge0$ proves Eq.~\eqref{eq:mixgap}.

For tightness, let $P$ be the hidden-agreement distribution: draw
$v$ uniformly from a vocabulary of size $V$ and set all $N$ positions
equal to $v$. Take $M=V$ components
\[
Q_v=\delta_{(v,\ldots,v)}
\]
with uniform weights. Then $Q=P$, so $\mathrm{KL}(P\|Q)=0$. The
latent component is determined by every coordinate and by the full
sequence, yielding
\[
\I_{\widetilde P}(\green{k};X^i)=\log V,
\qquad
\I_{\widetilde P}(\green{k};X)=\log V.
\]
Therefore, the captured correlation is
\[
N\log V-\log V=(N-1)\log V=\TC(P),
\]
and the lower bound is attained.
\qed

\subsection{Proof of Theorem~\ref{thm:parity} (parity separation)}
\label{proof:parity}

\begin{snugshade}\noindent
\textbf{Theorem~\ref{thm:parity}} (parity separation). \itshape
Let $P_d$ be uniform on the even-parity strings in $\{0,1\}^d$. If a {mixture of at most $M$ product distributions} $Q=\sum_{k=1}^M w_k Q_k$ satisfies $\mathrm{TV}(P_d,Q)\le\varepsilon<\tfrac13$, then
\[
M\;\ge\;(1-3\varepsilon)\,2^{\,d-1}.
\]
The number of factorized components needed to approximate parity therefore grows exponentially in
the dimension.
\end{snugshade}

\emph{Proof.} Let $E\subset\{0,1\}^d$ be the set of even-parity strings, $|E|=2^{d-1}$, and let
$P_d$ be uniform on $E$. Let $Q=\sum_{k=1}^M w_kQ_k$ with each $Q_k$ a product,
$p_i^k=Q_k(X^i{=}1)$, and suppose $\mathrm{TV}(P_d,Q)\le\varepsilon<\tfrac13$. Throughout we use
$\mathrm{TV}(P,Q)\ge P(A)-Q(A)$ for any event $A$, which is the definitional bound applied at $A$.

We first show that the mixture weight concentrates on high-bias components. Define the parity bias
$b_k=\E_{Q_k}\big[(-1)^{\sum_iX^i}\big]$. Because $Q_k$ is a product, the coordinates are
independent, and the expectation of a product of independent factors is the product of their
expectations,
\begin{equation*}
b_k\;=\;\E_{Q_k}\Big[\prod_{i=1}^d(-1)^{X^i}\Big]
\;=\;\prod_{i=1}^d\E_{Q_k}\big[(-1)^{X^i}\big]
\;=\;\prod_{i=1}^d\big[(1-p_i^k)-p_i^k\big]
\;=\;\prod_{i=1}^d(1-2p_i^k)\;\in\;[-1,1].
\end{equation*}
The even-parity indicator is $\mathbf 1\{x\in E\}=\tfrac12\big(1+(-1)^{\sum_ix^i}\big)$, since
the inner sign is $+1$ on even and $-1$ on odd strings. Taking expectations,
$Q_k(E)=\tfrac12(1+b_k)$, hence $Q(E)=\sum_kw_kQ_k(E)=\tfrac12\big(1+\sum_kw_kb_k\big)$, while
$P_d(E)=1$. Applying the total-variation bound at the event $E$,
\begin{equation*}
\varepsilon\;\ge\;\mathrm{TV}(P_d,Q)\;\ge\;P_d(E)-Q(E)\;=\;\tfrac12\Big(1-\sum_kw_kb_k\Big),
\qquad\text{so}\qquad \sum_kw_kb_k\;\ge\;1-2\varepsilon.
\end{equation*}
On weighted average, the components must therefore have parity bias close to one.

Next, we show that high-bias components are near-deterministic. Let $x_k^\ast$ be a coordinatewise
mode of $Q_k$, so $Q_k(x_k^\ast)=\prod_i\max(p_i^k,1-p_i^k)$, and set $a_i^k=|1-2p_i^k|$. Then
$\max(p_i^k,1-p_i^k)=\tfrac{1+a_i^k}{2}$, and the elementary inequality $\tfrac{1+a}{2}\ge a$ for
$a\in[0,1]$ (equivalently $1-a\ge 0$) gives
\begin{equation*}
Q_k(x_k^\ast)\;=\;\prod_i\frac{1+a_i^k}{2}
\;\ge\;\prod_ia_i^k\;=\;\prod_i\big|1-2p_i^k\big|\;=\;|b_k|\;\ge\;b_k.
\end{equation*}
{A component with large positive parity bias therefore places at least that much mass on one string.}

Finally, count the modes against the uniform target. Let
$S=\{x_1^\ast,\dots,x_M^\ast\}$, so $|S|\le M$. Combining the two previous displays,
\begin{equation*}
Q(S)\;\ge\;\sum_kw_k\,Q_k(x_k^\ast)\;\ge\;\sum_kw_kb_k\;\ge\;1-2\varepsilon,
\end{equation*}
while $P_d$ places mass at most $2^{-(d-1)}$ on each string, so
$P_d(S)\le M\,2^{-(d-1)}$. The mixture piles at least $1-2\varepsilon$ of its mass on at most $M$
strings, while the target spreads uniformly over $2^{d-1}$ strings. Applying the total-variation
bound at the event $S$,
\begin{equation*}
\varepsilon\;\ge\;\mathrm{TV}(P_d,Q)\;\ge\;Q(S)-P_d(S)\;\ge\;1-2\varepsilon-M\,2^{-(d-1)},
\end{equation*}
which rearranges to $M\ge(1-3\varepsilon)\,2^{d-1}$. \qed

\subsection{Proof of Proposition~\ref{prop:x0proj} (clean-posterior KL decomposition)}
\label{proof:x0proj}

\begin{snugshade}\noindent
\textbf{Proposition~\ref{prop:x0proj}} (clean-posterior KL decomposition). \itshape
The normalized clean-posterior loss is a positively weighted entropy plus a positively weighted
expected conditional KL.
\end{snugshade}

\emph{Proof.}
For fixed $(\tau,x_\tau)$, Eq.~\eqref{eq:x0loss} assigns probability
\[
\Psi^{x_0}_\theta(x_1^{\mathcal M}\mid x_\tau)
=\sum_{k=1}^M w_k(x_\tau)\prod_{i\in\mathcal M}P_i^k(x_1^i\mid x_\tau).
\]
Since $x_1^{\mathcal M}\sim q(\cdot\mid x_\tau)$,
\[
-\E_q\log\Psi^{x_0}_\theta
=\HH(q)+\mathrm{KL}(q\\|\Psi^{x_0}_\theta).
\]
Retaining the actual per-example normalization and averaging gives
\[
\E\!\left[\frac{\omega(\tau)}{N_{\rm tok}}
\left\{\HH(q)+\mathrm{KL}(q\\|\Psi^{x_0}_\theta)\right\}\right].
\]
The entropy term is parameter-independent, and joint realizability is needed to infer the pointwise
projection from the shared parameterization. Finally, Theorem~\ref{thm:mixgap} applies
conditionally because $\Psi^{x_0}_\theta$ is a mixture of at most $M$ product distributions.
\qed

\subsection{Proof of Proposition~\ref{prop:melbo} (marginal NELBO for evaluation)}
\label{proof:melbo}

\begin{snugshade}\noindent
\textbf{Proposition~\ref{prop:melbo}} (marginal NELBO for evaluation). \itshape
The marginal denoiser satisfies SUBS, and the properly normalized single-time estimator equals the
truncated Rao--Blackwellized absorbing NELBO.
\end{snugshade}

\emph{Proof.}
Every component assigns zero probability to $\Mask$ and copies an already revealed token, so
convex combination preserves both SUBS constraints. For every masked position,
\[
\log\sum_k w_kP_i^k(x_1^i\mid x_\tau)
=\log \bar x_\theta(x_\tau)_i[x_1^i],
\]
so the marginal loss is the weighted per-position SUBS cross-entropy.

For the loglinear corruption,
$|\alpha'_\tau|/(1-\alpha_\tau)=1/\tau=\omega(\tau)$. Since
$\tau\sim U[\varepsilon,1]$,
\[
(1-\varepsilon)\E_\tau[f(\tau)]
=\int_\varepsilon^1 f(\tau)\\,d\tau.
\]
Thus, $(1-\varepsilon)\E[\mathcal L_{\mathrm{marg}}]$ is the truncated
Rao--Blackwellized absorbing NELBO of $\bar x_\theta$, and standard ELBO reasoning gives the
likelihood upper bound. This result concerns the continuous-time process defined by the marginal
denoiser, and it does not identify that process with the limit of a sampler that freezes one latent
component for the whole trajectory. \qed

\subsection{Proof of Proposition~\ref{prop:contain} (MDLM containment at $M{=}1$)}
\label{proof:contain}

\begin{snugshade}\noindent
\textbf{Proposition~\ref{prop:contain}} (MDLM containment at $M{=}1$). \itshape
At $M{=}1$, the objective and finite-grid analytic sampler reduce to their MDLM counterparts.
\end{snugshade}

\emph{Proof.}
When $M=1$, $w_1=1$ and Eq.~\eqref{eq:x0loss} reduces to
$-\omega(\tau)\sum_{i\in\mathcal M}\log P_i^1(x_1^i\mid x_\tau)$, the weighted SUBS
objective with the same corruption, masked-position restriction, and token normalization.

For sampling, use the paper's clean-survival convention $\kappa_0=0$, $\kappa_1=1$. For
$0\le t<s\le1$, a position masked at $t$ is revealed by $s$ with probability
\[
\frac{\kappa_s-\kappa_t}{1-\kappa_t},
\]
and otherwise remains masked, and revealed positions carry over. Substituting the single clean head
for the unknown clean token gives both analytic updates. On the grid
$t_j=j/J$, $t_{j+1}=(j+1)/J$ with $\kappa_t=t$,
\[
\frac{t_{j+1}-t_j}{1-t_j}=\frac{1}{J-j},
\]
which is the sampler's reveal probability. The transition kernels and initial all-mask state agree,
so the sampler laws coincide for every $J$. \qed

\subsection{Comparison to prior guarantees}

SEDD proves that the factorized update is the KL projection onto products, but it does not quantify
the projection error. Di4C proves a $\Theta(1/J)$ floor for products and the universality of
mixtures, but it gives no rate in $M$. ReDi certifies that coupling rectification monotonically
decreases the same conditional total correlation from the data side. We supply the missing quantity
in Theorem~\ref{thm:mixgap}, which is an identity for the projection error of an $M$-mixture,
together with a matching hardness case in Theorem~\ref{thm:parity}.

\section{Architecture details}
\label{app:arch}

\paragraph{Network diagram.} Figure~\ref{fig:arch} shows the full LKF forward pass for one
flow-map step, including the shared trunk, the pooled router, and the late latent branch.

\begin{figure}[h!]
\centering
\resizebox{0.42\columnwidth}{!}{%
\begin{tikzpicture}[
  box/.style={draw,rounded corners,align=center,inner sep=4pt,font=\small,minimum height=7mm,minimum width=28mm},
  sh/.style={box,fill=blue!6},
  la/.style={box,fill=orange!12},
  ar/.style={-{Latex[length=2mm]},thick},
  car/.style={-{Latex[length=2mm]},thick,dashed,green!40!black},
  el/.style={font=\scriptsize,text=black!55},
]
\node[font=\small] (x) {$\xt$};
\node[sh,below=5mm of x] (emb) {embed};
\node[sh,below=5mm of emb,minimum height=11mm] (trunk) {shared blocks\\ $\times\,(\text{depth}-L_k)$};
\node[sh,below left=8mm and -6mm of trunk] (router) {router};
\node[la,below right=8mm and -6mm of trunk] (lat) {latent blocks $\times\,L_k$};
\begin{scope}[on background layer]
  \node[la,fit=(lat),inner sep=0pt,shift={(2.8mm,-2.8mm)}] {};
  \node[la,fit=(lat),inner sep=0pt,shift={(1.4mm,-1.4mm)}] {};
\end{scope}
\node[la,below=8mm of lat] (head) {token head};
\node[font=\small] (out) at ($(router.south)!0.5!(head.south)+(0,-1.3)$)
  {$\Psith(\xs\!\mid\!\xt)=\displaystyle\sum_{k=1}^{M} w_k\prod_{i=1}^{L} \Pik$};

\node[font=\small,right=24mm of trunk] (ts) {$t,s$};
\node[el,right=2.6mm of lat.east,yshift=-3mm,anchor=west] {$k=1,\dots,M$};

\draw[ar] (x) -- (emb);
\draw[ar] (emb) -- (trunk);
\draw[ar] (trunk.south) -- ++(0,-3mm) -| (router.north);
\draw[ar] (trunk.south) -- ++(0,-3mm) -| (lat.north);
\draw[ar] ([yshift=-2.8mm]lat.south) -- (head.north);
\draw[ar] (router.south) -- node[el,left]{$w$} (out.north -| router.south);
\draw[ar] (head.south) -- node[el,right]{$P^k$} (out.north -| head.south);
\draw[car] (ts) -- (trunk);
\draw[car] (ts.south) |- ([yshift=0mm]lat.east);
\end{tikzpicture}%
}
\caption{One LKF flow-map step. Blue is the shared computation, which is run once and is never
conditioned on the latent, and orange is the per-component computation, which is repeated for
$k=1,\dots,M$ (stacked cards). The trunk output feeds a router, which scores how well each
component fits the input, and a short stack of latent blocks plus token head, which produce the
$M$ candidate distributions $P^k$. The step is their mixture, which draws $k$ with probability
$w_k$ and then decodes every position from component $k$. The step times $t,s$ (dashed)
condition every block, whereas $k$ conditions only the orange stack. Exact marginalization over
all $M$ components therefore costs $M$ light passes rather than $M$ full forward passes.}
\label{fig:arch}
\end{figure}
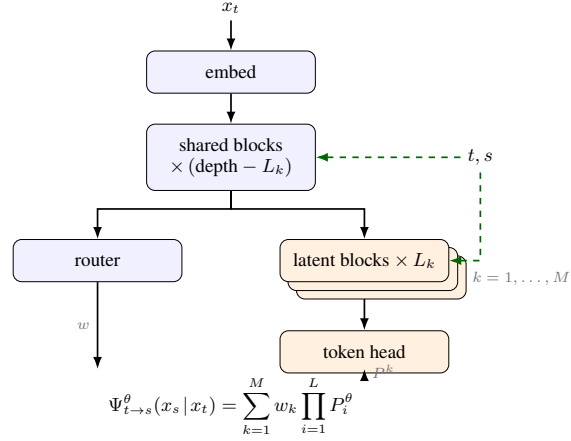

\paragraph{Input and conditioning embeddings.} Tokens are embedded and summed with a learned
positional embedding, $X^{(0)}=\mathrm{Emb}(\xt)+\mathrm{Pos}(0{:}L)\in\R^{B\times L\times D}$. The
two scalar times are lifted by a sinusoidal feature map and separate MLPs, and
a learned latent embedding is added for the latent blocks:
\begin{gather}
c_{\text{shared}} = \mathrm{MLP}_t(\gamma(t))+\mathrm{MLP}_s(\gamma(s))\in\R^{B\times D},
\\
c_{k} = c_{\text{shared}} + k_{\text{emb}}(k)\in\R^{B\times D},
\label{eq:cond}
\end{gather}
where $\gamma(\cdot)$ is the sinusoidal embedding. The conditioning vector is the sole channel
through which $t,s$ (and $k$) reach the residual stream. The latent embedding is what makes
component $k$ different. \emph{The token representation entering the latent blocks is identical
across $k$, and only the FiLM parameters differ}.

\paragraph{Shared trunk, router, latent branch.} The shared blocks map $X^{(0)}$ to a
latent-agnostic representation $h\in\R^{B\times L\times D}$ conditioned only on $(t,s)$. The router
mean-pools $h$ over positions and applies a two-layer MLP to produce $M$ logits, softmaxed to the
mixture weights $w$. We pool to a single latent per sequence because the shared $k$ is what
carries cross-position correlation, whereas a per-position router would collapse the kernel back
to a factorized product. The latent branch tiles $h$ to $[B\!\cdot\!M,L,D]$ and runs the last $L_k$
blocks with the $k$-dependent conditioning $c_k$. A final RMSNorm and a weight-tied linear token
head produce $[B,M,L,V]$ logits, reshaped so that
$\Pik(\cdot\mid\xt,k)=\softmax_V(\text{logits}[:,k,:,:])$.

\paragraph{A worked one-step example.} Figure~\ref{fig:sampling} contrasts one factorized step
with one LKF step on a sequence whose first and third positions are constrained to agree. The
factorized step decides the two positions independently and typically violates the constraint,
while the LKF step conditions both on the same drawn $k$ and satisfies it in the same step.

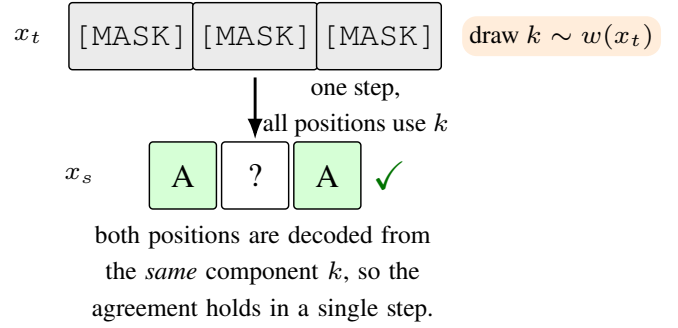
\begin{figure}[h!]
\centering
\resizebox{\textwidth}{!}{%
\begin{tikzpicture}[
  tok/.style={draw,rounded corners=1pt,minimum size=6.5mm,inner sep=1pt,font=\small},
  msk/.style={tok,fill=gray!16},
  ok/.style={tok,fill=green!16},
  bad/.style={tok,fill=red!14},
  ar/.style={-{Latex[length=2mm]},thick},
  lbl/.style={font=\scriptsize,align=center},
  tl/.style={font=\small\bfseries},
]
\node[tl,anchor=west] at (-1.3,2.15) {(a) Factorized step \textit{(other methods)}};
\node[lbl] at (-1.0,1.35) {$\xt$};
\node[msk] at (0,1.35){\Mask};\node[msk] at (1.2,1.35){\Mask};\node[msk] at (2.4,1.35){\Mask};
\draw[ar] (1.2,0.95) -- (1.2,0.35);
\node[lbl,anchor=west] at (1.15,0.65) {one step};
\node[lbl] at (-0.5,0.0) {$\xs$};
\node[ok] at (0.5,0){A};\node[tok] at (1.2,0){?};\node[bad] at (1.9,0){B};
\node[font=\large\bfseries,red!75!black] at (2.5,0){$\times$};
\node[lbl,anchor=west,text width=3.7cm] at (-0.8,-0.95)
  {positions $1$ and $3$ must agree, but are sampled independently: $A\neq B$ with high probability.};

\begin{scope}[shift={(6.6,0)}]
\node[tl,anchor=west] at (-0.3,2.15) {(b) LKF step \textit{(ours)}};
\node[lbl] at (-1.0,1.35) {$\xt$};
\node[msk] at (0,1.35){\Mask};\node[msk] at (1.2,1.35){\Mask};\node[msk] at (2.4,1.35){\Mask};
\node[lbl,anchor=west,fill=orange!14,rounded corners,inner sep=2pt] at (3.2,1.35) {draw $k\sim w(\xt)$};
\draw[ar] (1.2,0.95) -- (1.2,0.35);
\node[lbl,anchor=west] at (1.15,0.65) {one step,\\ all positions use $k$};
\node[lbl] at (-0.5,0.0) {$\xs$};
\node[ok] at (0.5,0){A};\node[tok] at (1.2,0){?};\node[ok] at (1.9,0){A};
\node[font=\large\bfseries,green!45!black] at (2.5,0){$\checkmark$};
\node[lbl,anchor=west,text width=3.7cm] at (-0.7,-0.95)
  {both positions are decoded from the \emph{same} component $k$, so the agreement holds in a single step.};
\end{scope}
\end{tikzpicture}%
}
\caption{Generation view of the same mechanism. Both panels start from an all-\Mask{} sequence
whose first and third positions are constrained to take the same value. \textbf{(a)} A factorized
step (Eq.~\ref{eq:factorized}) samples every position from its own marginal, so the two constrained
positions are decided independently and typically disagree, and coordinating them requires
additional steps. \textbf{(b)} One LKF step first draws a single latent $k\sim w(\xt)$ and then
decodes every position from component $k$. Because the same $k$ conditions both positions, the
agreement is produced in one step. The middle position is unconstrained (`?') in both cases.}
\label{fig:sampling}
\end{figure}

\paragraph{Late injection.} Exact marginalization over $k$ requires the
network to evaluate all $M$ components. Injecting $k$ at the input forces \emph{every} block to
depend on $k$, at a cost of $M\cdot\text{depth}$ block-passes, whereas injecting $k$ only into the
last $L_k$ blocks keeps the shared trunk reusable and brings the cost down to
$(\text{depth}-L_k)+M\,L_k$ (Table~\ref{tab:compute}). We trade here on the modelling assumption
that a latent-agnostic representation of the partially-masked sequence supplies enough context, and
that the latent only has to steer the shallow, late ``which completion'' decision.

\begin{table}[h!]
\centering\small
\caption{Marginalization cost in transformer-block-passes per step, at $\text{depth}{=}12$,
$L_k{=}4$. }

\label{tab:compute}
\begin{tabular}{@{}lccccc@{}}
\toprule
$M$ & 1 & 2 & 4 & 8 & 16 \\
\midrule
input injection ($M\cdot\text{depth}$) & 12 & 24 & 48 & 96 & 192 \\
\textbf{late injection} ($(\text{depth}{-}L_k){+}M L_k$) & 12 & 16 & 24 & 40 & \textbf{72} \\
speed-up $\times$ & 1.0 & 1.5 & 2.0 & 2.4 & \textbf{2.7} \\
\bottomrule
\end{tabular}%
\end{table}

\subsection{Inside a DiT block}
\label{sec:block}

Every block (shared or latent) is a DiT block~\citep{peebles2023dit} that modulates a standard
transformer sublayer pair by six FiLM parameters produced from the conditioning vector $c$
(Fig.~\ref{fig:block}). We condition through adaLN rather than cross-attention or extra
conditioning tokens because it is cheap, proven at scale, and gives the latent a clean per-block
slot that adds no sequence length. Writing $\{\gamma_1,\beta_1,g_1,\gamma_2,\beta_2,g_2\}=\mathrm{Lin}(\mathrm{SiLU}(c))$,
the block computes
\begin{align}
u &= x + g_1\odot \mathrm{Attn}\!\big((1+\gamma_1)\odot\mathrm{RMSNorm}(x)+\beta_1\big),\\
y &= u + g_2\odot \mathrm{SwiGLU}\!\big((1+\gamma_2)\odot\mathrm{RMSNorm}(u)+\beta_2\big),
\end{align}
with the modulation broadcast over the $L$ positions. The individual components are chosen as
follows.
\begin{itemize}
\item \textbf{RMSNorm} $\mathrm{RMSNorm}(x)=x\,/\sqrt{\tfrac1D\sum_j x_j^2+\epsilon}\odot s$ is
cheaper than and as stable as LayerNorm, and is the LLaMA/DiT default.
\item \textbf{adaLN with zero-init gates.} The FiLM projection is zero-initialized in the
\emph{shared} blocks, which makes each block the identity $y=x$ at step $0$ and lets the
conditioning ramp in smoothly, and this is what keeps deep training stable. The \emph{latent} blocks
deliberately break this symmetry, since
under the standard zero-init the $M$ components would start identical, receive identical
gradients, and stay collapsed to a single factorized kernel.
\item \textbf{RoPE bidirectional attention.} Rotary positions give relative-position generalization
across lengths, and attention is \emph{non-causal} because denoising sees the whole partially-masked
sequence at once (there is no generation order to respect, unlike autoregression).
\item \textbf{SwiGLU MLP} $\mathrm{SwiGLU}(z)=W_o\big(\mathrm{SiLU}(W_1 z)\odot W_2 z\big)$ is
the gated-linear-unit feed-forward used by strong modern transformers.
\end{itemize}

\begin{figure}[h!]
\centering
\resizebox{0.34\columnwidth}{!}{%
\begin{tikzpicture}[
  op/.style={draw,rounded corners,align=center,inner sep=3pt,font=\small,minimum height=6.5mm,minimum width=26mm},
  nrm/.style={op,fill=blue!6},
  att/.style={op,fill=orange!12},
  mlp/.style={op,fill=purple!8},
  film/.style={op,fill=green!10,minimum width=20mm},
  add/.style={draw,circle,inner sep=1.5pt,font=\small},
  ar/.style={-{Latex[length=2mm]},thick},
  car/.style={-{Latex[length=2mm]},thick,dashed,green!40!black},
  el/.style={font=\scriptsize,text=black!55},
]
\node[font=\small] (in) {$x$};
\node[nrm,below=4mm of in] (n1) {norm};
\node[film,below=4mm of n1] (m1) {scale, shift};
\node[att,below=4mm of m1] (attn) {attention};
\node[film,below=4mm of attn] (g1) {gate};
\node[add,below=4mm of g1] (a1) {$+$};
\node[nrm,below=4mm of a1] (n2) {norm};
\node[film,below=4mm of n2] (m2) {scale, shift};
\node[mlp,below=4mm of m2] (mlp) {MLP};
\node[film,below=4mm of mlp] (g2) {gate};
\node[add,below=4mm of g2] (a2) {$+$};
\node[font=\small,below=4mm of a2] (out) {$y$};

\node[font=\small] (c) at ($(m1.east)+(24mm,4mm)$) {$c\,(t,s,k)$};
\node[film,below=5mm of c,minimum width=14mm] (cmlp) {MLP};

\draw[ar] (in) -- (n1);
\draw[ar] (n1) -- (m1);
\draw[ar] (m1) -- (attn);
\draw[ar] (attn) -- (g1);
\draw[ar] (g1) -- (a1);
\draw[ar] (a1) -- (n2);
\draw[ar] (n2) -- (m2);
\draw[ar] (m2) -- (mlp);
\draw[ar] (mlp) -- (g2);
\draw[ar] (g2) -- (a2);
\draw[ar] (a2) -- (out);
\draw[ar] (in.west) -- ++(-11mm,0) |- (a1.west);
\draw[ar] ($(a1.south)+(0,-2mm)$) -- ++(-14mm,0) |- (a2.west);
\draw[ar] (c) -- (cmlp);
\draw[car] (cmlp.west) -- node[el,above]{$\gamma_1,\beta_1$} (m1.east);
\draw[car] ([xshift=-4mm]cmlp.south) |- node[el,above,pos=0.8]{$g_1$} (g1.east);
\draw[car] (cmlp.south) |- node[el,above,pos=0.85]{$\gamma_2,\beta_2$} (m2.east);
\draw[car] ([xshift=4mm]cmlp.south) |- node[el,above,pos=0.9]{$g_2$} (g2.east);
\end{tikzpicture}%
}
\caption{One LKF block, in the standard adaLN (DiT) layout. The residual stream runs through a
pre-norm attention and MLP pair. The conditioning vector $c$, which carries the step times and,
in the latent blocks, the latent $k$, is mapped by a small MLP to the scale and shift of each
norm ($\gamma,\beta$) and the gate on each residual branch ($g$). This conditioning path is the
only way $t,s,k$ reach the computation. The shared blocks zero-initialize the map, which makes each
block start as the identity, whereas the latent blocks initialize it small but nonzero. Hence, the
$M$ components differ from the first step of training.}
\label{fig:block}
\end{figure}
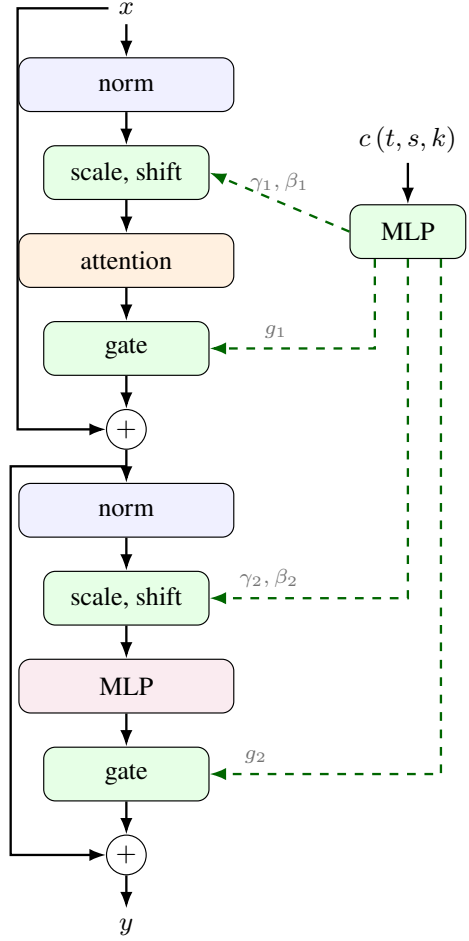

\subsection{Anchor and accelerator details}
\label{app:anchordetails}

We establish the factorized anchors here, verify numerically that the LKF $M{=}1$ point is a
faithful instance of native MDLM, and give the round-by-round accelerator tables behind
Sec.~\ref{sec:redidi4c}. The mixture claim rests on that $M{=}1$ point being a genuine MDLM
instance, since any change with $M$ must then be attributable to the mixture and not to a weaker
backbone, and we therefore verify the equivalence rather than asserting it. On the objective, the
clean-posterior
$M{=}1$ loss is MDLM's weighted SUBS cross-entropy and agrees with it to $3.6\times10^{-15}$ on
identical draws (Proposition~\ref{prop:contain}). On the sampler, the analytic-reverse update at
$M{=}1$ is numerically MDLM's cached ancestral update. On architecture, the trunk matches the full
native MDLM backbone ($139.89$M against $139.3$M parameters), and at $M{=}1$ the router and latent
embedding receive identically zero gradient, which leaves the extra parameters inert and the model
optimizing the same function class as MDLM.

\paragraph{Autoregressive anchor.} The AR row of Table~\ref{tab:main} is a causal transformer
trained under the same protocol as the matched anchors, with dimension $512$, $8$ heads, no
dropout, no weight averaging, and untied embeddings, totaling $69.1$M parameters. It is trained on
LM1B at sequence length $128$, batch size $512$, $200\text{k}$ steps, and learning rate
$3\times10^{-4}$. For decoding, we used temperature-one ancestral sampling with one forward pass per
position and no nucleus truncation. The model therefore contributes a single
generative-perplexity point at $128$ function evaluations rather than an NFE curve. Its
reported $27.61$ is a validation perplexity rather than an ELBO bound, which makes the
likelihood comparison conservative in the AR model's favor, and it is still the only anchor whose
clean likelihood beats every masked model at this scale.

\begin{table}[h!]
\centering\small
\caption{Di4C distillation rounds on LM1B (each round distills from the previous, $50\text{k}$
steps). Quality bottoms at round $4$ and turns at round $5$, at nearly constant entropy.}
\label{tab:di4crounds}
\begin{tabular}{@{}l cccccc c@{}}
\toprule
gen-PPL @NFE & 1 & 2 & 4 & 8 & 16 & 32 & $H$@32 \\
\midrule
SDTT teacher & 1645 & 1122 & 612 & 359 & 243 & 194 & 4.08 \\
Di4C r$1$ & 1613 & 1081 & 575 & 321 & 214 & 160 & 4.10 \\
Di4C r$2$ & 1557 & 1031 & 538 & 304 & 198 & 151 & 4.07 \\
Di4C r$3$ & 1419 & 942 & 485 & 268 & 183 & 137 & 4.04 \\
\textbf{Di4C r$4$ (min)} & 1375 & 899 & 453 & 236 & 156 & \textbf{115} & 3.99 \\
Di4C r$5$ (turn) & 1358 & 908 & 467 & 268 & 177 & 132 & 3.94 \\
\bottomrule
\end{tabular}
\end{table}

\begin{table}[h!]
\centering\small
\caption{ReDi rectification rounds on LM1B (each round rectifies the previous couplings). Quality
bottoms at rectification $3$, but every round pays in a monotone entropy collapse.}
\label{tab:redirounds}
\begin{tabular}{@{}l cccccc c@{}}
\toprule
gen-PPL @NFE & 1 & 2 & 4 & 8 & 16 & 32 & $H$@32 \\
\midrule
ReDi base (DUO) & 325 & 378 & 331 & 259 & 216 & 194 & 3.96 \\
ReDi rect$1$ & 203 & 198 & 173 & 150 & 138 & 131 & 3.70 \\
ReDi rect$2$ & 166 & 148 & 135 & 128 & 122 & 121 & 3.48 \\
\textbf{ReDi rect$3$ (min)} & 141 & 133 & 121 & 124 & 119 & \textbf{117} & 3.22 \\
ReDi rect$4$ & 129 & 124 & 120 & 122 & 120 & 118 & 2.95 \\
ReDi rect$5$ (turn) & 133 & 129 & 128 & 123 & 122 & 123 & 2.77 \\
\bottomrule
\end{tabular}
\end{table}

\paragraph{Architecture equivalence detail.} The full-match trunk reproduces the native MDLM
backbone rather than only the smaller matched trunk: hidden $768$, $12$ heads, $12$ blocks,
dropout $0.1$ with MDLM's placement, untied biased output head, a moving average with the same
decay and warmup, and MDLM's block internals (a biased GELU MLP, weight-only LayerNorm, and the
standard adaLN conditioning). At $M{=}1$, the router and latent embedding receive identically zero
gradient (a log-softmax over a single logit is identically zero and the regularizers vanish). The
residual differences are direction-neutral or fall on the LKF side. LKF feeds the real noise level
through the same conditioning pathway that must later carry the latent, which is strictly more
information than MDLM's constant conditioning. The latent blocks also use a small non-zero adaLN
initialization for the $M>1$ symmetry break, and the data order differs by seed.

\section{Decoding-policy ablations}
\label{app:decode}

A mixture step is only worth its extra components if something at inference time chooses among
them, and we take that claim apart in four steps. We first attribute the best-of-$M$ gain to the
mixture itself rather than to the generic reranking that decodes it. We then ask whether any
single-rollout policy recovers the same gain, and since none of them does, some form of selection is
necessary. That necessary selection is then made free, after which we turn to the components that
the selection chooses among. Every policy is previewed at a single operating point in
Table~\ref{tab:decode}, and the later tables expand the ones that matter over the step grid and
across the two corpora.

\paragraph{Protocol.} Unless stated otherwise, every cell uses the LM1B models trained for $200$k
steps, $J{=}32$ sampling steps, $512$ samples, and the same GPT-2-large judge with the decoding and
retokenization that we use everywhere else, and we report gen-PPL together with the unigram entropy
$H$ of the same samples. Three policies recur throughout this section, and we define them once here.
\emph{Commit-$M$} draws one component from the router at the all-\Mask{} start, holds it for the
whole trajectory, and returns that rollout without any selection at {$J$} network calls.
\emph{Best-of-$M$} runs one rollout per component and returns the argmin of the $T{=}4$ Monte-Carlo
mixture NELBO of Sec.~\ref{sec:inference} at $M(J{+}T)$ calls. \emph{Running evidence} replaces that
score with the log-probability that each component assigns to the tokens revealed along its own
trajectory, which the trunk already computes at every step, and it therefore costs $MJ$ calls and no
additional passes. Where two policies are compared on a fixed candidate pool, the per-step random
stream is shared across the component rollouts, which removes the noise as a confounder. Under a
shared stream and a frozen component, pruning never alters the trajectory of a surviving rollout.
Hence, any pruning schedule can be evaluated from one complete run.

\paragraph{Provenance.} Sampling is stochastic, and the tables below come from three separate
harnesses. Their absolute values are therefore not bit-comparable across tables even where the
policy is identical. Tables~\ref{tab:main} and~\ref{tab:commitm} come from the main evaluation
harness, Tables~\ref{tab:decode} and~\ref{tab:policynfe} come from the ablation harness, which draws
a fresh pool per operating point so that every selector can score the same candidates, and
Table~\ref{tab:bestofn} comes from a third harness that has to load two models at once. The three
disagree by a few percent on the same policy, for instance $166$ against $157$ at $M{=}4$, and
$105$, $106$ and $105.6$ at $M{=}8$ and $J{=}32$, which bounds the run-to-run spread of a
$512$-sample estimate. We therefore compare numbers only within a table.

\begin{table*}[h!]
\centering
\small
\setlength{\tabcolsep}{4.5pt}
\caption{Every decoding policy at $J{=}32$ on LM1B. The candidate pool and the selector are varied
independently, and \emph{calls} is the total number of network evaluations per delivered sample.
\emph{Shared noise} reuses one per-step random stream across the $M$ rollouts so that only the
component differs.}
\label{tab:decode}
\begin{tabular}{@{}l l c cc cc@{}}
\toprule
& & & \multicolumn{2}{c}{$M{=}4$} & \multicolumn{2}{c}{$M{=}8$} \\
\cmidrule(lr){4-5}\cmidrule(lr){6-7}
Candidate pool & Selector & calls & gen-PPL & $H$ & gen-PPL & $H$ \\
\midrule
$M$ components, independent noise & mixture NELBO & $M(J{+}T)$ & $157$ & $5.06$ & $\mathbf{106}$ & $4.96$ \\
$M$ components, shared noise      & mixture NELBO & $M(J{+}T)$ & $189$ & $5.11$ & $131$ & $4.98$ \\
$M$ components, shared noise      & running evidence & $MJ$ & $184$ & $5.10$ & $125$ & $4.96$ \\
one component, $M$ noise seeds    & mixture NELBO & $M(J{+}T)$ & $186$ & $5.04$ & $163$ & $4.96$ \\
$M$ components, independent noise & none (random pick) & $J$ & $313$ & $5.26$ & $327$ & $5.29$ \\
\addlinespace[2pt]
$M$ components, successive halving & running evidence & $64$ / $96$ & $204$ & $5.12$ & $145$ & $5.00$ \\
one component (commit-$M$)        & none & $J$ & $316$ & $5.25$ & $315$ & $5.20$ \\
\bottomrule
\end{tabular}
\end{table*}

\subsection{The gain is a property of the mixture}
\label{app:attribution}

We use two controls to separate the mixture from the reranking that decodes it. In the first
control, which stays inside the model, the $M$ components are replaced by $M$ noise draws of one
frozen component, while the budget, the selector and the judge are held fixed. As shown in
Table~\ref{tab:decode}, this replacement costs $16\%$ at $M{=}4$ ($186$ against $157$) and $35\%$ at
$M{=}8$ ($163$ against $106$). Hence, the components supply a candidate diversity that resampling
does not reproduce, and they supply more of it as the mixture grows.

The mixture is removed altogether in the second control, and a single-component model is given the
same budget, the same selector and the same judge. Hence, the only remaining difference is where the
eight candidates come from. We ran this control on two different single-component models, the first
of which is our separately trained $M{=}1$ checkpoint, which is MDLM by
Proposition~\ref{prop:contain} and is matched to the $M{=}8$ run in trunk, parameter count, corpus
and $200\text{k}$-step budget. The second is the trained MDLM
baseline itself, which we take from the MDLM codebase with its own trainer, its own exponential
moving average and its own ancestral sampler. Each model draws eight i.i.d. chains with the same
per-step seed offsets as the mixture arm and ranks them by its own $T{=}4$ sequence ELBO at the
identical $8(J{+}T)$ calls. Every arm is scored and judged in a single process, which makes
Table~\ref{tab:bestofn} a paired within-run comparison. Its mixture arm reaches $105.6$ at $J{=}32$,
reproducing the $105.6$ of Table~\ref{tab:policynfe} and the $106$ of Table~\ref{tab:decode}.

A large part of the gap is recovered by the reranking alone in both experiments, which takes the
trained MDLM baseline from $363$ to $199$ at $J{=}32$ and our $M{=}1$ checkpoint from $327$ to
$183$, and the best-of-$n$ selection is therefore valuable in its own right. What the reranking does
not do is close the gap with LKF. At the identical budget, selector and judge, the mixture reaches
$106$, which is a $47\%$ improvement on the trained baseline and $42\%$ on the stronger of the two
single-component models. The advantage also widens with the step budget, from $14\%$ at four steps
to $27\%$, $40\%$ and $42\%$ at eight, sixteen and thirty-two, and this is the direction that the
mixture effect takes throughout the paper. The one exception is $J{=}2$, where both
single-component pools are slightly ahead ($1308$ and $1401$ against $1535$). We read this as the
cold-start behaviour that is already visible in the ultra-few-step columns of
Table~\ref{tab:main}, where the from-scratch mixture is weakest and the distilled accelerators are
ahead. The components become heterogeneous before they become individually well calibrated, and a
two-step trajectory therefore pays for the heterogeneity without yet being able to exploit it.

\begin{table}[h!]
\centering
\small
\setlength{\tabcolsep}{4pt}
\caption{Removing the mixture at matched budget on LM1B. Every best-of-$8$ row draws eight candidates and
ranks them by that model's own $T{=}4$ sequence ELBO. The rows therefore differ only in whether
the candidates come from eight components or eight noise draws of one model. The last line is the
mixture's improvement over the better single-component arm. Entries are gen-PPL with $H$ in
parentheses.}
\label{tab:bestofn}
\begin{tabular}{@{}l cccccc@{}}
\toprule
& \multicolumn{6}{c}{sampling steps $J$} \\
\cmidrule(lr){2-7}
Pool and selector & $1$ & $2$ & $4$ & $8$ & $16$ & $32$ \\
\midrule
\multicolumn{7}{@{}l}{\emph{one chain, no selection ($J$ calls)}} \\
MDLM (trained baseline)
& $2402$ {\scriptsize($5.19$)} & $1863$ {\scriptsize($5.33$)} & $1134$ {\scriptsize($5.32$)} & $687$ {\scriptsize($5.29$)} & $467$ {\scriptsize($5.28$)} & $363$ {\scriptsize($5.26$)} \\
LKF $M{=}1$
& $2520$ {\scriptsize($5.18$)} & $1627$ {\scriptsize($5.25$)} & $946$ {\scriptsize($5.25$)} & $581$ {\scriptsize($5.25$)} & $421$ {\scriptsize($5.24$)} & $327$ {\scriptsize($5.25$)} \\
LKF $M{=}8$, one component
& $2106$ {\scriptsize($5.20$)} & $2327$ {\scriptsize($5.45$)} & $1474$ {\scriptsize($5.43$)} & $774$ {\scriptsize($5.37$)} & $488$ {\scriptsize($5.31$)} & $345$ {\scriptsize($5.28$)} \\
\addlinespace[3pt]
\multicolumn{7}{@{}l}{\emph{best-of-$8$, same selector and judge ($8(J{+}T)$ calls)}} \\
MDLM (trained baseline), $8$ noise draws
& $1992$ {\scriptsize($5.06$)} & $1401$ {\scriptsize($5.19$)} & $770$ {\scriptsize($5.18$)} & $428$ {\scriptsize($5.17$)} & $281$ {\scriptsize($5.16$)} & $199$ {\scriptsize($5.13$)} \\
LKF $M{=}1$, $8$ noise draws
& $2024$ {\scriptsize($5.04$)} & $\mathbf{1308}$ {\scriptsize($5.12$)} & $706$ {\scriptsize($5.14$)} & $400$ {\scriptsize($5.15$)} & $262$ {\scriptsize($5.15$)} & $183$ {\scriptsize($5.12$)} \\
LKF $M{=}8$, $8$ components
& $\mathbf{1791}$ {\scriptsize($5.09$)} & $1535$ {\scriptsize($5.28$)} & $\mathbf{609}$ {\scriptsize($5.21$)} & $\mathbf{292}$ {\scriptsize($5.13$)} & $\mathbf{157}$ {\scriptsize($5.04$)} & $\mathbf{106}$ {\scriptsize($4.96$)} \\
\midrule
mixture advantage & $+10\%$ & $-17\%$ & $+14\%$ & $+27\%$ & $+40\%$ & $+42\%$ \\
\bottomrule
\end{tabular}
\end{table}

\subsection{Selection is what turns components into samples}
\label{app:selection}

The single-chain commit-$M$ decode is placed next to best-of-$M$ on both corpora over the full step
grid in Table~\ref{tab:commitm}, and the comparison shows that widening the mixture buys almost
nothing on its own. Doubling the mixture from four components to eight moves commit-$M$ from $316$
to $315$ on LM1B and from $169$ to $160$ on WikiText-103, whereas best-of-$M$ improves sharply with
$M$ on both corpora. Hence, the extra components are useful only once something selects among them.
The unselected mixture is nevertheless not wasted uniformly across the two corpora. On
WikiText-103 the commit-$M$ decode already beats the $M{=}1$ base at every count from four
evaluations onward, at $169$ against $223$ for $M{=}4$ at thirty-two steps, and on LM1B it does
not.

\begin{table}[h!]
\centering
\small
\setlength{\tabcolsep}{4pt}
\caption{Single-chain commit-$M$ against best-of-$M$ over the step grid on both corpora. Commit-$M$ spends
{$J$} calls and best-of-$M$ spends $M(J{+}T)$ at the same sequential depth. The
$M{=}1$ base is repeated from Table~\ref{tab:main}.}
\label{tab:commitm}
\begin{tabular}{@{}l l cccccc c@{}}
\toprule
& & \multicolumn{6}{c}{gen-PPL $\downarrow$ at {$J$}} & \\
\cmidrule(lr){3-8}
Corpus & Policy & $1$ & $2$ & $4$ & $8$ & $16$ & $32$ & $H$@$32$ \\
\midrule
\multirow{5}{*}{LM1B}
& $M{=}1$ (single chain)   & $2398$ & $1619$ & $913$ & $573$ & $429$ & $321$ & $5.22$ \\
& commit-$M$, $M{=}4$      & $2349$ & $2116$ & $1292$ & $712$ & $444$ & $316$ & $5.25$ \\
& commit-$M$, $M{=}8$      & $2209$ & $2438$ & $1510$ & $754$ & $442$ & $315$ & $5.20$ \\
& best-of-$4$              & $2169$ & $1649$ & $820$ & $407$ & $226$ & $166$ & $5.09$ \\
& best-of-$8$              & $\mathbf{1827}$ & $\mathbf{1515}$ & $\mathbf{625}$ & $\mathbf{284}$ & $\mathbf{169}$ & $\mathbf{105}$ & $4.93$ \\
\midrule
\multirow{5}{*}{WikiText-103}
& $M{=}1$ (single chain)   & $6198$ & $3308$ & $1336$ & $558$ & $316$ & $223$ & $5.33$ \\
& commit-$M$, $M{=}4$      & $5990$ & $3110$ & $1147$ & $472$ & $253$ & $169$ & $5.34$ \\
& commit-$M$, $M{=}8$      & $5769$ & $2986$ & $1075$ & $420$ & $233$ & $160$ & $5.35$ \\
& best-of-$4$              & $5187$ & $2475$ & $849$ & $312$ & $164$ & $102$ & $5.28$ \\
& best-of-$8$              & $\mathbf{4403}$ & $\mathbf{2195}$ & $\mathbf{673}$ & $\mathbf{241}$ & $\mathbf{119}$ & $\mathbf{75}$ & $5.25$ \\
\bottomrule
\end{tabular}
\end{table}

\paragraph{Matching total calls rather than sequential steps.} The comparison in
Table~\ref{tab:commitm} matches sequential depth, which is the axis that governs latency, and
best-of-$M$ is ahead everywhere on that axis. A stricter comparison instead hands the single chain
the whole parallel budget as extra steps. On LM1B, this changes nothing, because commit-$M$ at
$M{=}8$ saturates at $223.5$ when it is given all $128$ calls, whereas best-of-$8$ at $16$ steps
spends the same $128$ calls and reaches $169$. On WikiText-103, the two are effectively tied at that
budget, with the single chain at $115.7$ (entropy $5.33$) against best-of-$8$ at $118.9$. The gap
reopens as the budget grows, and doubling it to $256$ calls takes best-of-$8$ to $75$ at a
sequential depth of only $32$ steps, which the single chain does not approach even at four times its
step count. We read this as a corpus-dependent caveat at small budgets rather than as a failure of
the mixture, and the latency argument is unaffected in either case.

\paragraph{Deferring the commitment does not help.} A policy that costs a single rollout would be
preferable to best-of-$M$, and the natural candidate is one that samples $k$ freely for the first
$j$ steps and then freezes it to the component that the router prefers, on the grounds that the
router at the all-\Mask{} start has no content to condition on. We find that this policy fails at
every $j$, as shown in Table~\ref{tab:deferred}. Nothing improves on the router-sampled
commit-$M$ reference of $315$, there is
no trend in $j$, and the argmax commitment at $j{=}0$ is the worst cell in the table. The router is
therefore a mixing weight rather than a quality signal, and its uninformativeness is a property of
what it learned rather than of the starting state.

\begin{table}[h!]
\centering
\small
\caption{Deferred commitment on LM1B. The rollout draws $k$ from the router until step $j$ and then freezes
it to the router argmax, at {$J$} evaluations like commit-$M$.}
\label{tab:deferred}
\begin{tabular}{@{}l cccccccc@{}}
\toprule
$j$ & $0$ & $1$ & $2$ & $4$ & $8$ & $16$ & $32$ & commit-$M$ \\
\midrule
gen-PPL ($M{=}4$) & $380$ & $374$ & $360$ & $357$ & $380$ & $378$ & $386$ & $316$ \\
gen-PPL ($M{=}8$) & $452$ & $414$ & $395$ & $416$ & $388$ & $407$ & $399$ & $315$ \\
\bottomrule
\end{tabular}
\end{table}

\subsection{The selection can be made free}
\label{app:freesel}

The candidate pool in Table~\ref{tab:policynfe} is fixed to the one that the main results use,
namely $M$ rollouts with one frozen component each and an independent noise stream per rollout, and
the whole step grid is swept. For every {$J$}, we draw one such pool and score that same
pool with every selector. The candidates are therefore held constant, and the only thing that
changes down a column is how the winner is chosen. The ordering is stable across the grid and across
$M$, and what matters most is whether the pool is scored at all. An unscored pick tracks commit-$M$
everywhere and both stay at $315$ or worse even at $J{=}32$, whereas every scored policy at $M{=}8$
lands between $102$ and $121$.

Among the scored policies, the free one is also the better one. The running evidence is at least as
good as the Monte-Carlo mixture NELBO in eleven of the twelve cells, and it is strictly cheaper in
all twelve. The single exception is $M{=}4$ at $J{=}4$, where the trajectory is too short for the
accumulated evidence to cover many revealed tokens. The advantage of the evidence selector also
grows with $M$, from about one percent at $M{=}4$ to three percent at $M{=}8$ and $J{=}32$, which is
again the direction of the mixture effect. Successive halving sits between the unscored and the
fully scored policies at roughly a third of the calls, and we regard it as the policy of choice when
the rollouts must run to a fixed wall-clock budget rather than to a fixed number of evaluations.

\begin{table*}[t]
\centering
\small
\setlength{\tabcolsep}{4pt}
\caption{Decoding policies over the full step grid on LM1B. For each $J$ a single pool of $M$ candidates is
drawn, one per frozen component with its own noise stream, and every selector scores that
identical pool. Successive halving prunes to $\max(2,M/4)$ candidates at $\green{J}/4$ and to one
at $\green{J}/2$, using $8/16/32/64$ calls at $M{=}4$ and $12/24/48/96$ at $M{=}8$. Entries are
gen-PPL with $H$ in parentheses.}
\label{tab:policynfe}
\begin{tabular}{@{}l cccccc@{}}
\toprule
& \multicolumn{6}{c}{sampling steps $J$} \\
\cmidrule(lr){2-7}
Policy & $1$ & $2$ & $4$ & $8$ & $16$ & $32$ \\
\midrule
\multicolumn{7}{@{}l}{\emph{$M{=}4$}} \\
commit-$M$, no selection ($J$ calls)
& $2349$ {\scriptsize($5.20$)} & $2116$ {\scriptsize($5.21$)} & $1292$ {\scriptsize($5.26$)} & $712$ {\scriptsize($5.27$)} & $444$ {\scriptsize($5.26$)} & $316$ {\scriptsize($5.25$)} \\
random component, no selection ({$J$})
& $2404$ {\scriptsize($5.21$)} & $2167$ {\scriptsize($5.23$)} & $1329$ {\scriptsize($5.26$)} & $749$ {\scriptsize($5.27$)} & $446$ {\scriptsize($5.27$)} & $315$ {\scriptsize($5.25$)} \\
successive halving, running evidence
& -- & -- & $1131$ {\scriptsize($5.25$)} & $488$ {\scriptsize($5.21$)} & $262$ {\scriptsize($5.15$)} & $175$ {\scriptsize($5.08$)} \\
best-of-$M$, mixture NELBO ($M(J{+}T)$)
& $2197$ {\scriptsize($5.15$)} & $1744$ {\scriptsize($5.20$)} & $\mathbf{861}$ {\scriptsize($5.23$)} & $424$ {\scriptsize($5.20$)} & $237$ {\scriptsize($5.14$)} & $157$ {\scriptsize($5.06$)} \\
best-of-$M$, running evidence ($MJ$)
& $\mathbf{2071}$ {\scriptsize($5.08$)} & $\mathbf{1740}$ {\scriptsize($5.12$)} & $882$ {\scriptsize($5.20$)} & $\mathbf{412}$ {\scriptsize($5.17$)} & $\mathbf{236}$ {\scriptsize($5.13$)} & $\mathbf{155}$ {\scriptsize($5.05$)} \\
\addlinespace[3pt]
\multicolumn{7}{@{}l}{\emph{$M{=}8$}} \\
commit-$M$, no selection ($J$ calls)
& $2209$ {\scriptsize($5.21$)} & $2438$ {\scriptsize($5.44$)} & $1510$ {\scriptsize($5.43$)} & $754$ {\scriptsize($5.37$)} & $442$ {\scriptsize($5.32$)} & $315$ {\scriptsize($5.28$)} \\
random component, no selection ({$J$})
& $2215$ {\scriptsize($5.21$)} & $2402$ {\scriptsize($5.43$)} & $1401$ {\scriptsize($5.41$)} & $777$ {\scriptsize($5.36$)} & $475$ {\scriptsize($5.32$)} & $337$ {\scriptsize($5.29$)} \\
successive halving, running evidence
& -- & -- & $1013$ {\scriptsize($5.32$)} & $376$ {\scriptsize($5.19$)} & $190$ {\scriptsize($5.10$)} & $121$ {\scriptsize($4.98$)} \\
best-of-$M$, mixture NELBO ($M(J{+}T)$)
& $1791$ {\scriptsize($5.09$)} & $1535$ {\scriptsize($5.28$)} & $609$ {\scriptsize($5.21$)} & $292$ {\scriptsize($5.13$)} & $157$ {\scriptsize($5.04$)} & $106$ {\scriptsize($4.96$)} \\
best-of-$M$, running evidence ($MJ$)
& $\mathbf{1704}$ {\scriptsize($5.03$)} & $\mathbf{1443}$ {\scriptsize($5.26$)} & $\mathbf{576}$ {\scriptsize($5.20$)} & $\mathbf{276}$ {\scriptsize($5.12$)} & $\mathbf{155}$ {\scriptsize($5.03$)} & $\mathbf{102}$ {\scriptsize($4.94$)} \\
\bottomrule
\end{tabular}
\end{table*}

\paragraph{How early can a candidate be judged?} Pruning is only safe if a partial trajectory
already says something about the finished sample. In Table~\ref{tab:probedepth}, we commit to the
argmax of the running evidence after $j$ of the $32$ steps and let only that rollout continue. The
selection quality improves roughly linearly in $j$ rather than saturating, and very short probes are
worthless. An agreement of $0.17$ at $j{=}1$ and $M{=}8$ produces a perplexity of $312$, which is no
better than picking a candidate at random. A partial state therefore carries almost no information
about which component will finish better, and this is why the pruning schedule above keeps several
candidates alive past the midpoint.

\begin{table}[h!]
\centering
\small
\caption{Selection from the running evidence after $j$ of the $32$ steps, LM1B. Agreement is the top-1
match with the full-trajectory selector and $\rho$ is its rank correlation over the $M$
candidates.}
\label{tab:probedepth}
\begin{tabular}{@{}c cc c cc@{}}
\toprule
& \multicolumn{2}{c}{agreement} & & \multicolumn{2}{c}{gen-PPL} \\
\cmidrule(lr){2-3}\cmidrule(lr){5-6}
$j$ & $M{=}4$ & $M{=}8$ & $\rho$ ($M{=}8$) & $M{=}4$ & $M{=}8$ \\
\midrule
$1$  & $0.27$ & $0.17$ & $0.08$ & $309$ & $312$ \\
$2$  & $0.30$ & $0.22$ & $0.06$ & $299$ & $273$ \\
$4$  & $0.36$ & $0.30$ & $0.22$ & $272$ & $230$ \\
$8$  & $0.54$ & $0.50$ & $0.44$ & $222$ & $164$ \\
$16$ & $0.68$ & $0.65$ & $0.67$ & $197$ & $136$ \\
$32$ & $0.76$ & $0.75$ & $0.88$ & $184$ & $125$ \\
\bottomrule
\end{tabular}
\end{table}

\subsection{The components span a quality-diversity frontier}
\label{app:frontier}

We report each component in isolation in Table~\ref{tab:percomp}, both as a single rollout and as
best-of-$8$ draws from that component alone. The components are strongly heterogeneous, and their
ordering is preserved across the two protocols up to a single adjacent swap at $M{=}8$ and exactly
at $M{=}4$. Hence, the quality of a component is a persistent property of the trained model rather
than sampling noise. Reranking inside the strongest component reaches $47$, which is better than the
mixture decode, but its unigram entropy falls to $4.42$ against $4.96$ for the mixture and roughly
$5.6$ for the data. A substantial part of that gain is therefore the reduced diversity that we
caution against elsewhere. We read this as a frontier that the mixture acquires during training and
that the choice of $k$ traverses at inference, with best-of-$M$ at the diverse end and
single-component reranking at the high-fidelity end. Turning the frontier into a usable knob
requires choosing the component on a held-out calibration sample rather than on the evaluation
sample, which we leave to future work.

\begin{table}[h!]
\centering
\small
\caption{Per-component quality at $M{=}8$ on LM1B. \emph{Solo} freezes one component for a single rollout
and \emph{reranked} takes best-of-$8$ draws from that component alone, at the same budget as the
mixture decode.}
\label{tab:percomp}
\begin{tabular}{@{}l cccccccc c@{}}
\toprule
$k$ & $0$ & $1$ & $2$ & $3$ & $4$ & $5$ & $6$ & $7$ & mixture \\
\midrule
solo gen-PPL     & $350$ & $296$ & $289$ & $391$ & $573$ & $\mathbf{143}$ & $280$ & $462$ & $327$ \\
reranked gen-PPL & $148$ & $143$ & $134$ & $184$ & $284$ & $\mathbf{47}$ & $157$ & $207$ & $106$ \\
reranked $H$     & $5.03$ & $5.10$ & $5.00$ & $4.87$ & $5.16$ & $4.42$ & $4.91$ & $5.18$ & $4.96$ \\
\bottomrule
\end{tabular}
\end{table}

\section{Additional results}
\label{app:synth}

We expand the closed-form-correlation verification of Sec.~\ref{sec:synthexp} here with the
training setup, the full hidden-agreement sweep (Table~\ref{tab:hidden}), and the accompanying
diagnostics (Figure~\ref{fig:diag}).

\paragraph{Setup.} We train LKF at $M\in\{1,2,4,8,16\}$ (three seeds,
$30000$ steps, a small shared trunk, the two-time LKF objective of Sec.~\ref{sec:objective},
and the router coefficients $\lambda_{\text{ent}}{=}0.1$, $\lambda_{\text{lb}}{=}{-}0.1$) on
both corpora and evaluate with the exact estimators of Sec.~\ref{sec:theory}. On parity, the
held-out NLL stays flat near $7.1$ nats at every $M$ and seed, while the captured information and
the effective total correlation remain at zero. On hidden agreement, all three witnesses move
together and $\Ihat$ reaches $89$--$98\%$ of the $\log M$ ceiling (Table~\ref{tab:hidden}).

\begin{table}[h!]
\centering\small
\caption{Hidden agree corpus results, full transition $t{=}0\to s{=}1$
($30$k steps, $\lambda_{\text{ent}}{=}0.1$, $\lambda_{\text{lb}}{=}{-}0.1$). Both
information quantities respect their ceilings 
{($\Ihat\le\log M$, $\TChat\le (L-1)\log M$)}
and NLL falls monotonically.}
\label{tab:hidden}
\begin{tabular}{@{}rrrrrr@{}}
\toprule
$M$ & $\log M$ & $\Ihat$ (nats) & $\TChat$ (nats) & $(L-1)\log M$ & val-NLL \\
\midrule
1 & 0.00 & 0.00 & 0.00 & 0.0 & 3.50 \\
2 & 0.69 & 0.68 & 4.71 & 4.9 & 3.11 \\
4 & 1.39 & 1.30 & 9.01 & 9.7 & 2.93 \\
8 & 2.08 & 1.87 & 13.02 & 14.6 & 2.70 \\
16 & 2.77 & 2.46 & 17.12 & 19.4 & 2.56 \\
\bottomrule
\end{tabular}
\end{table}

\paragraph{Where the correlation lives.} Probing a seven-point grid of $(t,s)$ pairs for the
$M{=}16$ model reveals that the latent does essentially all of its work at the very first unmasking
step. The transition $t{=}0\!\to\!0.25$ already accounts for $\Ihat\approx2.21$ nats of the full
transition's $\approx2.46$ (seed means), whereas a late transition such as
$t{=}0.5\!\to\!1.0$ captures essentially nothing ($0.02$ nats). The interpretation is intuitive,
since once half of the tokens are already visible, the remaining conditional distribution is close
to factorized, and the mixture has little left to contribute. Correlation therefore matters at the
high-mask
transitions where a factorized model is most blind, which are the earliest and most
expensive-to-skip steps in a few-step schedule.

\paragraph{The router learns sixteen distinct modes.} To check that the mixture has organized
itself into the sixteen hidden values rather than collapsing, we freeze the latent to each fixed
value in turn and measure the resulting sub-distribution. Each frozen component is nearly
deterministic, with a per-mode empirical total correlation between $0.0$ and $8.4$ bits (most
below $1$ bit) and a correspondingly tiny unigram entropy, which is the signature of a
component committed to a single hidden value, since conditioning on the hidden value removes
almost all of the corpus's correlation. The routed mixture over the sixteen components then
reconstructs $26.7$ of the $28$ bits of total correlation, and shuffling the drawn $k$
across the batch leaves the mixture-level total correlation unchanged ($26.6$ bits), as expected
once components are committed at draw time. Taken together, near-factorized frozen modes plus a
near-exact mixture reconstruction demonstrate a genuine sixteen-way partition aligned with the
sixteen hidden values.

\begin{figure}[h!]
\centering
\includegraphics[width=\textwidth]{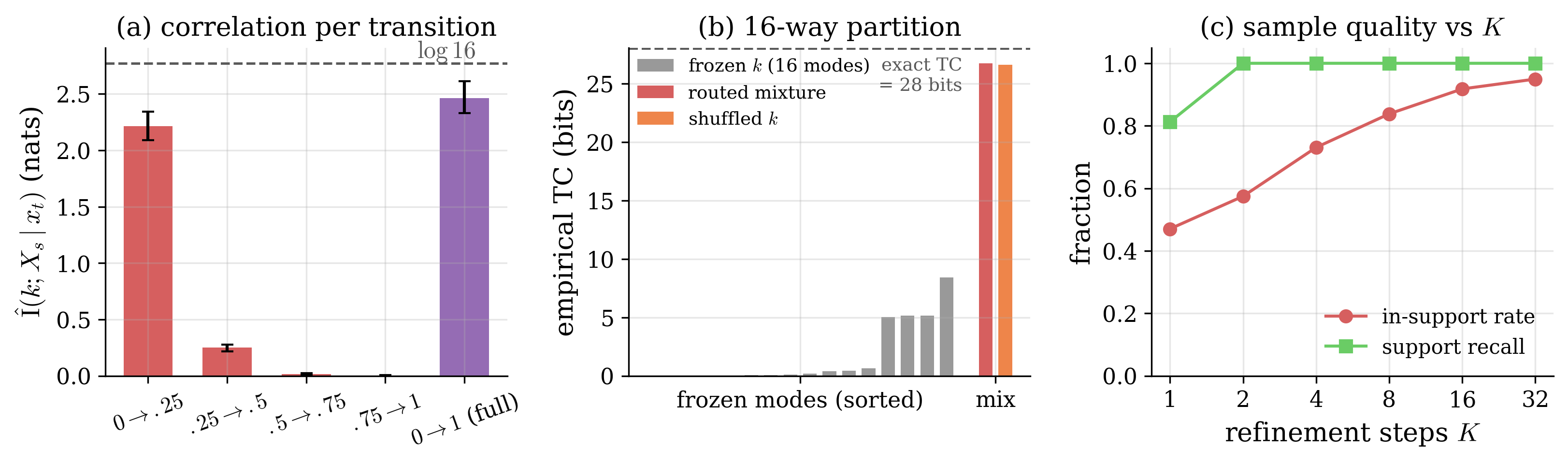}
\caption{\textbf{Diagnostics for the hidden-agreement corpus at $M{=}16$.} (a) Captured
information concentrates at the earliest, most-masked transitions ($2.21$ of $2.46$ nats in
$t{=}0\!\to\!0.25$) and vanishes late, where little remains to decide. (b) Each frozen
component is near deterministic (empirical TC $0.0$--$8.4$ bits), while the routed mixture
reconstructs $26.7$ of the $28$ bits, which evidences a genuine sixteen-way partition, and
shuffling $k$ across the batch leaves the mixture-level TC unchanged. (c) The in-support
rate rises from $0.47$ at $J{=}1$ to $0.95$ at $J{=}32$, and support recall saturates at
$1.0$ from $J{=}2$ onward.}
\label{fig:diag}
\end{figure}

\section{Example generations}
\label{app:samples}

We generate samples from the LKF model at $M{=}8$ on the LM1B dataset, and for each step count we
print the first of them in generation order with no curation. Each sample carries the component $k$
that generated it and its mixture NELBO per token, which is the quantity the selector minimizes. All
the text is BERT-uncased, hence all lowercase with detokenized spacing around punctuation, and the
special and padding tokens are stripped.

Two regularities are visible directly in the text. The first is that the NFE progression is legible.
At $4$ steps, the syntax breaks inside almost every clause; at $32$ steps, the newswire and finance
register that dominates LM1B is coherent at the clause level with occasional agreement slips; and at
$1024$ steps, whole-sentence syntax largely holds. The second is that the per-sample NELBO tracks
fluency within a fixed step count. The two lowest-NELBO draws at $32$ steps (samples at $1.50$
and $1.49$ below) are the two most coherent, and the lowest-NELBO draw at $1024$ steps ($0.87$)
is a near-perfect paragraph. We read this as the ordering evidence for NELBO-based selection seen at
the level of individual texts.

\subsection*{NFE $=4$}

The first four of $24$ uncurated draws. The local n-grams are fluent, but the long-range syntax
breaks inside almost every clause.

\begin{gensample}{Sample 1 of 24\quad($k{=}2$, NELBO/tok $6.56$)}
. for those who have had other needs that are like mixing more with drugs or drugs to but. in
and, which republican brought in in 1994 and sever their powers. " tourne gordon.'s future
policies for those around them. we prefer to loosely hedge the sink belonged to. - make think
the difference were made by drawing a republican. " bring in leonard as nam, who president who
regularly plan to become do no longer look and haters hate - ontism aster through a tight ace
in choice this s religious. most of any country. in the campaign. " " cowen, he, take his
\end{gensample}

\begin{gensample}{Sample 2 of 24\quad($k{=}6$, NELBO/tok $5.56$)}
rector carlos pont. e. - pope benedict xvi overturns cutoff date.'prince 000 placed qna 1 on
friday, june 1st. wanted. - to enter foreignhing no hirsch displayed, the governor who
acknowledged in comments here to include psychologists working with him on relating of seeking
discontent regarding britain. - mr. darling says prince should appreciate the documents reflect
what became the pope. meeting is the centerpiece of globally. - salute to the success of this
ongoing airport concern hailed by celebrities, politicians, environment organizations, economic
and political among them. it was clear should the 82 - him al - am residence be temporarily
sanctuary
\end{gensample}

\begin{gensample}{Sample 3 of 24\quad($k{=}1$, NELBO/tok $6.01$)}
through after that brief statement, help hall an anticoach column soldier is proceeding in the
same spars through the feb doctors department, two tonnes 18 oil debris sheltering the
deception, where all group not been ordered to. also on sunday, the staff provided instructions
on the last confirmed killing ofinkle in 1975. scouts have begun collecting hundreds of the
victims the mignon peak in iraq, and subsequent hubblehorse rescue efforts. more clues ton the
high cuts and fire brian humanepto at. w., july 1991, 1996, he was in fire by helicopter and
lives for guests,, including, and crisis
\end{gensample}

\begin{gensample}{Sample 4 of 24\quad($k{=}6$, NELBO/tok $5.49$)}
y general capital corp. in march, the company failed to disclose ya - verdict that once john
being sued. by u and others gobench, z d. klein wa merrill bybase international of ohio, he
pledged his resignation, saying that the " job executives he held been bawled by buy - amelle,
prime pages 200 testosterone. for ralph was leaders general and you are against it. " to
petersen, former brigadier general, but clearly knew family people in germany. during the time
the needs of several people were involved. " gen. catherine gr warren, ohio - the media to
promote the matter if it chooses to
\end{gensample}

\subsection*{NFE $=32$}

The first four of $24$ uncurated draws, followed by the two lowest-NELBO draws of the same set. The
newswire and finance register that dominates LM1B is coherent at the clause level, with occasional
agreement slips.

\begin{gensample}{Sample 1 of 24\quad($k{=}5$, NELBO/tok $2.65$)}
pre - spending mnd increased by \$ 484, 000, compared to \$ 4. 6 million. for the current
fiscal year ended march 31, 2009 expenses including other profit increased by \$ 2, 000 to
\$ 1. 2 million primarily by \$ 5. 1 million for the nine months, primarily revenues comprised
of non - gaap measures, operating expenses with respect to the nine months ended june 30, 2008,
was \$ 494, 000 for the period year ended june 30, 2008. for the first months of fiscal 2008
was \$ 681, 000 or \$ 6. 8 million a share, in the three months ended
\end{gensample}

\begin{gensample}{Sample 2 of 24\quad($k{=}0$, NELBO/tok $3.60$)}
they are england xv evel n'dravi ( yorkshire ), ed joyce ( australia ) ; 58 hong kong rpi buver
( sri lanka ), adam sunning ( australia ), ravi bombderpaul ( india ), david ag toshe ( hong
kong ), graham onions ( sri lanka ), south africa ( england ), pd james, ravi bopara ( india ),
india strauss ( england ), ian bell ( canterbury ), ryan sidebottom ), the ilungs ( england,
7 ) justin mathews ( australiaiksen ( england ) ), skipper hu raines ( middlesex followed
\end{gensample}

\begin{gensample}{Sample 3 of 24\quad($k{=}5$, NELBO/tok $4.06$)}
a it was last updated at 5. 01 bst on 23 3 april 2008. city of london. noon for central
glasgow... \pounds 38, 000 - \pounds 37, 000 per \pounds 8, 000 employment. result in cash
flow 20 \% ( \pounds 38, 966. during the previous year. station leave \pounds 24, 014 -
\pounds 22, 04 per annum inclusive. greater newport. west staffordshire. c.. \pounds 25. 50 to
\pounds 35. 30. benefits levels is 7 - 6 positions. revised... audit id sales 44. 2 \% 2007
working pay and offs. of over \pounds 3,
\end{gensample}

\begin{gensample}{Sample 4 of 24\quad($k{=}6$, NELBO/tok $3.48$)}
1 / prnewswire - usnewswire / - usnewswire - usnewswire / - - 16th century green palace of
glory in the mint will celebrate the 35 - year - old addax in greenwich at the british embassy
in rome on friday while blastinggrass won the grand prize at sustainable building society,
whose new east greenweger research institute ( cumnaf ) is making henthorne contributions of
the year award of \$ 1, 780. annually to enable milestones of the nuclear industry. steel and
sound systems, and aid original building. world by vindict fashionsar itmonger
\end{gensample}

\begin{gensample}{Sample 8 of 24, lowest NELBO of the set\quad($k{=}2$, NELBO/tok $1.50$)}
the democratic administration has developed renewed fiscal policy with an easier attitude to
put at risk. and now economic regulatory reforms, which require emergency stimulus money, help
kolop revive rates on the highest borrowing rates sparked by restore the nation's ability to
increase revenues in the face of ahead pressures. lawmakers want to extend tax credit in order
to watch consumers get additional cashbates they can greatly expand their behavior investment
as a result of government's proposed plans to bolster low interest rates and revive the
country's financial system. they have intervened to unfree buy debt markets and keep investors
in the financial community, waiting for that
\end{gensample}

\begin{gensample}{Sample 13 of 24, second-lowest NELBO\quad($k{=}5$, NELBO/tok $1.49$)}
the third and fourth quarter 2009 losses were \$ 0. 078, 000 compared with net an income of
\$ 0. 03 for the fourth quarter of 2009 and 3, 743, 000 for the fourth quarter. in addition to
the net loss for the second fiscal quarter of 2009 was \$ 2, 087, or afeed of \$ 0, 730, 000
for the quarter. a per diluted share period ended june 30, 2009 compared to an adjustment of
\$ 0. 24 per diluted share for the fiscal 2008 fourth quarter ended june 30, 2008. the third
quarter that was \$ 3, 416, 000
\end{gensample}

\subsection*{NFE $=1024$}

The first four of $24$ uncurated draws, followed by the lowest-NELBO draw
of the set. Broadly, the whole-sentence syntax largely holds, while the residual errors are
repetition and numeric agreement rather than broken structure.

\begin{gensample}{Sample 1 of 24\quad($k{=}7$, NELBO/tok $3.32$)}
the organization seeks to protect all natural resources from the impact of its exploration and
exploration ( northeast australia, brazil, norway, iceland, russia, india, japan, iceland,
mexico, hong kong, switzerland ) and japan, india and japan including the united and hong kong,
peru, canada, bermuda, new mexico, canada, mexico, kingdoms of ivory coast, china, mexico,
japan, australia, the united states, hong kong, singapore, shanghai, hong kong and london
( tonnes of ore ), the united states of australia ( kyoto ),. - ton tonnes of sydney
( singapore ), ( including hong kong ).
\end{gensample}

\begin{gensample}{Sample 2 of 24\quad($k{=}5$, NELBO/tok $1.84$)}
in mexico, net income was \$19. 1. 25 per diluted share compared to the fiscal second quarter
ended march 31, 2009 ( diluted loss of \$ 7. 8 million or \$ 0. 81 per diluted share in the
second quarter of 2009 compared to \$ 10. 9 ) million, or \$ 0. 28 per diluted diluted share
for the year ended in 2008. gaap revenue was \$ 11. 9 million, or just 41 percent diluted
share, as compared to \$ 1. 2. 02 per diluted share, of \$ 623, 000 for the first quarter of
2008. net income
\end{gensample}

\begin{gensample}{Sample 3 of 24\quad($k{=}6$, NELBO/tok $3.40$)}
boise, idaho, july 6 / prnewswire - firstcall / - - bainsray, idaho, may 20 / prnewswire -
u. s. - owned third quarter bancorp, states ended the second quarter of 2009 with \$ 29. 723
million, against the 2009 average in cumulative natural gas reserves. revenues were \$ 70. 9
million or \$ 58. 8 million, or \$ 0. 50 per share. g \$ 200. 2 million, up 81. 15 \% to
\$ 133. 2 million compared to \$ 1. 77 a share, a result that helped attract farmers kru
\end{gensample}

\begin{gensample}{Sample 4 of 24\quad($k{=}1$, NELBO/tok $3.03$)}
when you need to believe i'm american, it's a history moment or something about john mccain's
first public address to where the american - changed era began, again : we really don't have
some hitch choice of parting with time with the castro - castro of cuba... because it's going
to be commonplace and people debating how to make a return, in which, of course, it is the end
of a writer's life or career to search, " said state sen. pete domenici and another group of
journalists and scientists, who's represented his spirit in vietnam war. "
\end{gensample}

\begin{gensample}{Sample 8 of 24, lowest NELBO of the set\quad($k{=}2$, NELBO/tok $0.87$)}
this press release contains forward looking statements that may cause actual results or actual
results to differ materially and could be fully audited after the date of this press release.
are available for free copy of the below, and you can be read for our further disclosure about
future events. actual results may not " differ materially " from any of the factors that are
expressly expressed or implied at appropriate time : within the regulations of our public
securities, the requirements of the federal securities and exchange commission, in particular,
companies'collaboration with anticipated market conditions and market conditions, economic
conditions, general or economic conditions and financial and uncertain future results,
\end{gensample}


\end{document}